%% file: main.tex
\documentclass{article}

\usepackage[preprint]{neurips_2025}

\usepackage[utf8]{inputenc} % allow utf-8 input
\usepackage[T1]{fontenc}    % use 8-bit T1 fonts
\usepackage{hyperref}       % hyperlinks
\usepackage{url}            % simple URL typesetting
\usepackage{booktabs}       % professional-quality tables
\usepackage{amsfonts}       % blackboard math symbols
\usepackage{nicefrac}       % compact symbols for 1/2, etc.
\usepackage{microtype}      % microtypography
\usepackage{xcolor}         % colors

\usepackage{amsmath, amssymb, natbib, graphicx, url}
\usepackage{amssymb}
\usepackage{natbib}

\usepackage{bbm}
\usepackage{mathtools,  verbatim}
\usepackage{amsthm,amssymb}

\usepackage{wasysym}
\usepackage{hyperref}
\usepackage[symbol]{footmisc}

\usepackage{thm-restate}
\usepackage{thmtools}

\newtheorem{theorem}{Theorem}

\newtheorem{corollary}{Corollary}
\newtheorem{lemma}{Lemma}

\newtheorem{definition}{Definition}
\newtheorem{axiom}{Axiom}

\usepackage{comment}
\usepackage{algorithm}
\usepackage{algpseudocode}
\usepackage{bbm, dsfont}

\usepackage{graphicx}
\usepackage{subcaption}
\usepackage{booktabs}
\usepackage{tabularx}

\newcounter{protocol}
\newenvironment{protocol}[1]
  {\par\addvspace{\topsep}
   \noindent
   \tabularx{\linewidth}{@{} X @{}}
    \hline
    \refstepcounter{protocol}\textbf{Protocol \theprotocol} #1 \\
    \hline}
  { \\
    \hline
   \endtabularx
   \par\addvspace{\topsep}}

\newcommand{\reals}{\mathbb{R}}

\newcommand{\conv}{\textrm{conv}}

\newcommand{\Y}{\mathcal{Y}}

\newcommand{\N}{\mathbb{N}}

\newcommand{\relint}{\textrm{relint}}
\newcommand{\dom}{\textrm{dom}}

\newcommand{\ones}{\mathbbm{1}}

\DeclareMathOperator*{\argmax}{arg\,max}
\DeclareMathOperator*{\argmin}{arg\,min}

\DeclareUnicodeCharacter{03B2}{\ensuremath{\beta}}

\newcommand{\actionset}{\mathcal{C}}
\newcommand{\Reg}{\text{Reg}}

\newcommand{\ignore}[1]{}

\usepackage{times}
\usepackage{color}
\usepackage{soul}

\definecolor{darkblue}{rgb}{0.0,0.0,0.2}
\definecolor{darkgreen}{rgb}{0.0,0.3,0.0}
\hypersetup{colorlinks,breaklinks,
	linkcolor=darkblue,urlcolor=darkblue,
	anchorcolor=darkblue,citecolor=darkblue}
\usepackage[colorinlistoftodos,textsize=tiny]{todonotes} % need xargs

\newcommand{\Comments}{1}
\newcommand{\mynote}[2]{\ifnum\Comments=1\textcolor{#1}{#2}\fi}
\newcommand{\mytodo}[2]{\ifnum\Comments=1	\todo[linecolor=#1!80!black,backgroundcolor=#1,bordercolor=#1!80!black]{#2}\fi}

\newcommand{\rick}[1]{\mynote{red}{[RN: #1]}}
\newcommand{\rickt}[1]{\mytodo{red!20!white}{RN: #1}}

\newcommand{\journal}[1]{}%{\mytodo{gray!20!white}{BTW: #1}}%TURN OFF FOR NOW \mytodo{gray}{#1}}

\ifnum\Comments=1              % fix margins for todonotes
\fi

\title{Smooth Quadratic Prediction Markets}

\author{
Enrique Nueve\\
Department of Computer Science\\
University of Colorado Boulder\\
\texttt{enrique.nueveiv@colorado.edu} \\
\And
Bo Waggoner \\
Department of Computer Science\\
University of Colorado Boulder \\
\texttt{bwag@colorado.edu} \\ 
}

\begin{document}

\maketitle

\begin{abstract}
%The Duality-based Cost Function Market Maker (DCFMM) for Arrow-Debreu (AD) securities is analogous to Follow-The-Regularized-Leader \citep{abernethy2013efficient}. \bo{Suggest: When agents trade in a Duality-based Cost Function prediction market, they collectively implement the learning algorithm Follow-The-Regularized-Leader~\citep{abernethy2013efficient}.}
When agents trade in a Duality-based Cost Function prediction market, they collectively implement the learning algorithm Follow-The-Regularized-Leader~\citep{abernethy2013efficient}.
We ask whether other learning algorithms could be used to inspire the design of prediction markets.
%set prices in a prediction market while achieving a low worst-case monetary loss.
By decomposing and modifying the Duality-based Cost Function Market Maker's (DCFMM) pricing mechanism, we propose a new prediction market, called the Smooth Quadratic Prediction Market, the incentivizes agents to collectively implement general steepest gradient descent.
Relative to the DCFMM, the Smooth Quadratic Prediction Market has a better worst-case monetary loss for AD securities while preserving axiom guarantees such as the existence of instantaneous price, information incorporation, expressiveness, no arbitrage, and a form of incentive compatibility. 
To motivate the application of the Smooth Quadratic Prediction Market, we independently examine agents' trading behavior under two realistic constraints: bounded budgets and buy-only securities. 
Finally, we provide an introductory analysis of an approach to facilitate adaptive liquidity using the Smooth Quadratic Prediction Market.
Our results suggest future designs where the price update rule is separate from the fee structure, yet guarantees are preserved.
%Our results suggest that a learning algorithm can control the price update rule as fees are independently charged, all while preserving axiomatic guarantees for the prediction market.
%\bo{I think ``as fees are independently charged'' could be hard to interpret}
\end{abstract}
%%%%%%%%%%%%%%%%%%%%%%%%%
%%%%%%%%%%%%%%%%%%%%%%%%%

%%%%%%%%%%%%%%
\input{intro}

%%%%%%%%%%%%%%

%%%%%%%%%%%%%%
\input{background}
\input{quad_market}

%%%%%%%%%%%%%%

%%%%%%%%%%%%%%
\input{constraints}
%%%%%%%%%%%%%%

%%%%%%%%%%%%%%
\input{adaptive_liquid_new}

%%%%%%%%%%%%%%

\section{Conclusion}
For brevity, we refer the reader to Appendix \ref{app:conc} for a conclusion section. 

{\bf Broader Impacts:} Our work informs the design of prediction markets, and thus, our work may influence the choices of prediction market makers and traders. 
Due to the inherent risk of losing money in prediction markets, we acknowledge the possibility of negative impacts due to this work.

%%%%%%%%%%%%%%%%%%
%%%%%%%%%%%%%%%%%%

%%%%%%%%%%%%%%

\newpage 
\begin{ack}
We thank the following people for their feedback throughout this project: Stephen Becker for talks on gradient descent, Maneesha Papireddygari and Rafael Frongillo for talks about liquidity provisioning, Drona Khurana for talks about convergence results, and Aaron Sidford for notes on general steepest descent. 
\end{ack}
\bibliographystyle{plainnat}
\bibliography{bib}
%%%%%%%%%%%%%%

%\crefalias{section}{appendix} % uncomment if you are using cleveref
\newpage 
\appendix

\input{conc}\label{app:conc}

\input{apend}

%%%%%%%%%%%%%%%%%%%%%%%%%%%%%%%%%%%%%%%%%%%%%%%%%%%%%%%%%%%%

%%%%%%%%%%%%%%%%%
%%%%%%%%%%%%%%%%%

%%%%%%%%%%%%%%%%%
%%%%%%%%%%%%%%%%%

\end{document}

%% file: intro.tex
\section{Introduction}

\paragraph{What are Prediction Markets?} Prediction markets allow traders to buy and sell securities whose payoffs depend on the realization of future events \citep{hanson2003combinatorial}. 
Prediction market platforms such as Kalshi and Polymarket on average move millions of dollars per event-based market.
In this work, we examine prediction markets that trade Arrow-Debreu (AD) securities, which pay $\$1$ if a particular state of the world is reached and $\$0$ otherwise \citep{arrow1964role}.
Given a finite random variable $\Y =\{y_1, \dots ,y_d\}$ over $d$ mutually exclusive and exhaustive outcomes, AD securities are designed to elicit a full probability distribution with respect to $\Y$, reflected by the current prices of the securities.

\paragraph{Duality-based CFMM} The work of \cite{abernethy2013efficient} introduced a general framework for designing automated prediction markets over combinatorial or infinite state outcome spaces named the Duality-based Cost Function Market Maker (DCFMM).
As stated by \cite{abernethy2013efficient}, ``An automated market maker is a market institution that adaptively sets prices for each security and is always willing to accept trades at these prices.'' %\bo{is this sentence/quote needed?}
The DCFMM satisfies the desirable market axioms of no arbitrage, bounded worst-case loss, information incorporation, expressiveness, incentive compatibility, and efficient computability.
The DCFMM over the last decade has been examined in numerous works such as \cite{abernethy2014general,devanurbudget,frongillo2017axiomatic,frongillo2023axiomatic}.
A core observation of \cite{abernethy2013efficient} is that the DCFMM uses Follow-The-Regularized-Leader (FTRL) to set prices for securities, by interpreting past trades as loss vectors.
%However, although FTRL is used to set prices, the regret rate of FTRL and the worst-case loss of the DCFMM are not equivalent.
Although not our main contribution, as it has been hinted at in the literature before by \cite{devanurbudget}, to our knowledge, we are the first to formally show (in Appendix \ref{app:cftrlequiv}) a tight equivalence between Continuous FTRL and DCFMM in terms of regret and worst-case monetary loss of the mechanism. 
Inspired by this relation, we ask if other machine learning algorithms could be used to design a new prediction market with at least the same or better properties than the DCFMM.
This work proposes a new market maker payment mechanism called the Smooth Quadratic Prediction Market, which is analogous to the learning algorithm general steepest descent in terms of trader's incentives. 
%\bo{suggest: remove ``AD'' from the name everywhere}
%\bo{For space, the previous paragraph and the next could probably be combined and this one shortened.}

\paragraph{Smooth Quadratic Prediction Markets}
We propose the Smooth Quadratic Prediction Market, which increases profitability for the market maker relative to the DCFMM but still has the performance guarantees of the DCFMM, such as no arbitrage, bounded worst-case loss, information incorporation, expressiveness, computational efficiency, and a form of incentive compatibility.
%The pricing mechanism of DCFMM can be written as a linear term plus a Bregman divergence. 
%Using a fundamental inequality that states that the norm of a smooth function upper bounds the matching Bregman divergence term, we design the payment mechanism of the Smooth Quadratic AD Prediction Market. \bo{I think this is a bit too vague. 
We show that the Smooth Quadratic Market can be interpreted as replacing a Bregman divergence term in the DCFMM payments by a simpler quadratic ``fee''.
%Through this design, the market maker achieves larger payments on bundle purchases than the DCFMM. 
%Interestingly, the pricing scheme of the Smooth Quadratic AD Prediction Market is equivalent to general steepest descent w.r.t. the smooth cost function's norm. 
In the body of this work, we show that the Smooth Quadratic Prediction Market has many of the same axiom guarantees as that of the DCFMM, demonstrate how agents are incentivized to follow general steepest descent when trading, examine how agents trade under either bounded budgets or a buy-only market, and provide an approach to facilitate adaptive liquidity.

%% file: background.tex
\section{Background and Notation}
Let us denote the all ones vector by $\ones  = (1,\dots ,1 )\in\reals^d$.
Let us denote by the vector $\delta_i= (0,\dots ,1,\dots ,0)$, i.e. all-zero except for a one in the $i$-th position.
Comparison between vectors is pointwise, e.g. $q  \succ q'$ if $q_i>q_i'$ for all $i=1,\dots , d$ and similarly for $ \succeq$. We say $q\gneqq q'$ when $q_i\geq q_i'$ for all $i$ and $q\neq q'$.
We denote $((u)_+)_i = \max (u_i, 0)$.
For a set $Y \subseteq \reals^d$, we denote by $\text{closure}(Y)$ as the smallest closed set containing all of the limit points of $Y$. 
%We shall denote by $\text{conv}(S)$ the convex hull of $S$.
Define $\reals_{+}^{d}$ to be the nonnegative orthant.
Let $\Delta_d=\{ p\in \reals_{+}^{d} \mid \|p\| =1 \}$ be the set of probability distributions over $d$ outcomes, represented as vectors.
Let $\relint (\Delta_d)=\Delta_d \setminus \{p\in\Delta_d |\exists \; i\in [d], p_i=0 \}$.
We let $||\cdot || :\reals^{d}\to \reals_{+}$ denote a general p-norm. 
Given a norm $\|\cdot \|$, we define it's dual norm $\|y \|_{*}= \sup_{\|x \|\leq 1}\langle x,y \rangle$.
Let $f:\reals^{d} \to (-\infty ,+\infty ]$ be a function. 
We define the Fenchel conjugate of $f$ by $ f^*:\reals^d\to [-\infty ,+\infty ]$ such that $f^* (y)=\sup_{x\in\dom (f)}\;\langle y,x\rangle -f(x)$.
We will use the following definitions for a function $f$.
\begin{itemize}
    \item \textbf{convex}: $\forall \; x,y\in\reals^d$, $\lambda \in [0, 1]$, $f(\lambda x+(1-\lambda )y)\leq \lambda f(x)+(1-\lambda )f(y)$.
    \item \textbf{increasing}: $f (q) >  (q')$ $\forall$ $q, q'\in\reals^d$ with $q \gneqq q'$.
    \item \textbf{1-invariant}: $f(q+\alpha\ones )=f(q)+\alpha$ for $q\in\reals^d, \alpha\in \reals$. 
    \item \textbf{probability mapping}: $f$ is twice-differentiable, $\nabla f:\reals^{d}\to \Delta_{d}$, and $\text{closure}(\{\nabla f(q)\mid q\in\reals^{d}\} ) = \Delta_d$.
\end{itemize}
We shall refer to a function which is convex, increasing, 1-invariant, and a probability mapping as \textbf{CIIP}.
%Given a list of vectors, e.g. $h=(r_1,\dots ,r_t)$ the sum is denoted $\text{sum}(h) = r_1 + \dots + r_t$. 
%Concatenation of a new trade $r$ onto a history $h$ is denoted $h\oplus r$.

%%%%%%%%%%%%%%%%%%%%%

\subsection{Bregman Divergences, Smoothness, and Strong Convexity}
This section presents core concepts used throughout this work.  
We emphasize that we define both L-smoothness and K-strong convexity through a general norm $\|\cdot \|$, not just via the 2-norm, as is common in some literature.

%\begin{definition}[Subgradient]
%    For a proper function $f:\reals^{d} \to  (-\infty ,+\infty ] $, we define a subgradient of $f$ at $x \in\reals^{d}$ as a vector $g\in\reals^{d}$ that satisfies $f(y)\geq f(x) +\langle g,y-x\rangle$ $\forall y\in\reals^{d}$.
%    We let $\partial f(x)$ to denote the set of all subgradients of $f$ at $x$. 
%\end{definition}

%\begin{definition}[Bregman divergence]
%    For a differentiable function $f:\reals^{d} \to  (-\infty ,+\infty ] $, we define the Bregman divergence as $D_{f}(x,y)=f(x)-[f(y)+\langle \nabla f (y), x-y \rangle ]$. 
%\end{definition}

\begin{lemma}[\cite{hiriart2004fundamentals}, Proposition 6.1.1]\label{lemma:monograd}
    For a convex differentiable function $f:\reals^{d}
    \to (-\infty ,+\infty ]$, it holds that $\langle \nabla f(x)-\nabla f(y), x-y \rangle \geq 0$ $\forall$ $x,y\in\reals^{d}$.
\end{lemma}

%\begin{definition}[L-smoothness]
%  A differentiable function $f:\reals^{d} \to  (-\infty ,+\infty ]$ is L-smooth w.r.t. $||\cdot ||$, for some $L\geq 0$, i.f.f. $\forall \;x,y\in \reals^{d}$ it holds that $\|\nabla f(x)-\nabla f(y)\|_*\leq L \| x-y \|$.
%\end{definition}

%\begin{corollary}
%    By definition, if a function is $L$-smooth  with respect to $\|\cdot \|$, then it is $L'$-smooth for every $L'\geq L$. 
%\end{corollary}

%\begin{lemma}[\cite{balles2020geometry}, Lemma 1]\label{lem:bregsmooth}
%    If $f:\reals^{d} \to  (-\infty ,+\infty ]$ is convex, differentiable, and $L$-smooth w.r.t. $||\cdot ||$ then $\forall$ $x,y\in \reals^{d}$ it holds that
    %$$0\leq D_f(y,x)=f(y)-f(x)-\langle \nabla f(x),y-x\rangle \leq \frac{L}{2}\|x-y\|^{2}~.$$
%\end{lemma}

%\begin{definition}[K-strongly convex]
%Let $f:\reals^{d}\to (-\infty ,+\infty ]$ be a convex function and let $K>0$. 
%We say that $f$ is K-strongly convex w.r.t. $||\cdot ||$ if $\forall$ $x,y\in \reals^{d}$ and $\forall$ $\lambda \in [0,1]$ it holds that
%$$f(y)\geq f(x)+\langle \nabla f(x),y-x\rangle +\frac{K}{2}||y-x||^{2}~.$$
%\end{definition}
%\bo{As discussed, I'd recommend combining Definition 1, Lemma 2, and Definition 3 into one definition that defines both smoothness and strongly-convex in terms of the Bregman divergence.}

\begin{definition}[Bregman divergence, L-smoothness, \& K-strongly convex]\label{def:bregsmoothstrong}
For a differentiable function $f:\reals^{d} \to  (-\infty ,+\infty ] $, we define the Bregman divergence as $D_{f}(x,y)=f(x)-[f(y)+\langle \nabla f (y), x-y \rangle ]$. 
If $f:\reals^{d} \to  (-\infty ,+\infty ]$ is convex and differentiable, we say that $f$ is $L$-smooth w.r.t. $||\cdot ||$ if $\forall$ $x,y\in \reals^{d}$ it holds that $D_f(x,y)\leq \frac{L}{2}\|x-y\|^{2}$ where $L\geq 0$.
If $f:\reals^{d} \to  (-\infty ,+\infty ]$ is convex and differentiable, we say that $f$ is K-strongly convex w.r.t. $||\cdot ||$ if $\forall$ $x,y\in \reals^{d}$ it holds that $\frac{K}{2}\|x-y\|^{2}\leq D_f(x,y)$ where $K\geq 0$.

\end{definition}

Conceptually, K-strongly convex acts as a quadratic lower bound of a convex function, while L-smoothness serves as a quadratic upper bound of a convex function. 
Finally, we present a theorem demonstrating the dual relationship between strong convexity and smoothness.

%\begin{theorem}\cite{zhou2018fenchel}[Thm 1]\label{thm:strongsmoothdual}
%A function $f$ and its Fenchel conjugate function $f^*$ satisfy the following assertions:
%\begin{enumerate}
%\item If $f$ is closed and strong convex with parameter $K$, then $f^*$ has a Lipschitz continuous gradient with parameter $\frac{1}{K}$.
%\item If $f$ is convex and has a Lipschitz continuous gradient with parameter $L$, then $f^*$ is strong convex with parameter $\frac{1}{L}$.
%\end{enumerate}
%\end{theorem}

\begin{theorem}[\cite{kakade2009duality}, Theorem 6]\label{thm:dualsmoothstr}
Assume that f is a closed and convex function. Then $f$ is K-strongly convex w.r.t. a norm  $\|\cdot  \|$ i.f.f. $f^*$ is $\frac{1}{K}$-smooth
w.r.t. the dual norm $\|\cdot \|_{*}$.
\end{theorem}

\subsection{Automated Market Makers for AD Securities}  

We first introduce the general framework for an automated market maker.
\begin{definition}[Automated Market Maker for AD Securities]
  Say we have a finite random variable $\Y =\{y_1, \dots ,y_d\}$ over $d$ mutually exclusive and exhaustive outcomes.
  An automated market maker for AD securities, with initial state $q_0\in \reals^{d}$, operates as follows.  
   At round $t\in \mathbb{N}_0 $,
    \begin{enumerate}
       \item A trader can request any bundle of securities $r_t\in\reals^{d}$.
       \item The trader pays the market maker some amount $\text{Pay}(q_t,r_t)\in\reals$ in cash.
       \item The market state updates to $q_{t+1}=q_{t}+r_t$.
   \end{enumerate}
After an outcome of the form $Y=y_i$ occurs, for each round $t$, the trader responsible for the trade $r_t$ is paid $(r_t)_i$ in cash, i.e. the number of shares purchased in outcome $y_i$.
The market payout for the bundle $r_t$ and the outcome $Y=y$ is expressed via $\langle r_t,\rho (y) \rangle$ where $\rho :\Y \to \delta_{y}$. 
At any state $q_t$, the market maker can infer the belief of the market via the instsantaneous price $\text{InstPrice}(q_t) = f(q_t,0)$ where $f(q_t,r_t) = \nabla_{r_t} \text{Pay}(q_t,r_t)$.
\end{definition}

%By the Cost Function Market Maker being defined via $C$, which is \textbf{CIIP}, given state $q_t$, we can infer a belief of the market via $p_t=\nabla C(q_t)$, which is also referred to as the instantaneous price.
%For Cost Function Market Makers, we propose a novel decomposition called the Linear and Fee Market Decomposition, which will later be used to simplify our analysis.

Intuitively, the instantaneous price of security $i$ is the amount that would be paid for an arbitrarily small amount of security $i$.
We now define a particular family of automated market makers, which captures the DCFMM and the Smooth Quadratic Prediction Market.

%\begin{definition}[Linear and Fee Market Decomposition]\label{def:lfdecomp}
%    Let $C:\reals^{d}\to\reals$ be \textbf{CIIP} and assume that the market is at state $q_t\in\reals^d$. 
%    Let the price for a bundle $r_t\in \reals^{d}$ cost a linear term $\langle \nabla C(q_t),r_t \rangle$ plus a fee term $\text{Fee}(q_t,r_t)\in\reals$.
%    In general, we can denote the payment to the market via the cost of a bundle $r_t$ and instantaneous price $p_t =\nabla C(q_t)$ by 
%    $$\text{Pay}(q_t,r_t)= \langle \nabla C(q_t),r_t \rangle +\text{Fee}(q_t,r_t)= \langle p_t,r_t \rangle 
%+\text{Fee}(q_t,r_t)~.$$
%\end{definition}

\begin{definition}[Price-Plus-Fee Market]\label{def:lfdecomp}
    A Price-Plus-Fee Market is an automated market maker of the following form.
    Let $C:\reals^{d}\to\reals$ be \textbf{CIIP}. 
    Then
    $$\text{Pay}(q_t,r_t)= \langle \nabla C(q_t),r_t \rangle +\text{Fee}(q_t,r_t)= \langle p_t,r_t \rangle 
+\text{Fee}(q_t,r_t)~$$
    where $\text{Fee}(q_t,r_t) = o(\|r_t\|)$.
    We note that in this case, $\text{InstPrice}(q_t) = p_t = \nabla C(q_t)$.
\end{definition}

Given that both markets we examine in this work are Price-Plus-Fee markets, we shall use $\nabla C(q_t)$ and $\text{InstPrice}(q_t)$ interchangeably.
%\bo{Not a big deal either way, but since we went to the trouble of defining InstPrice, it would be nice to use it in the axioms. We could say that we use them interchangeably.}

Generally, it is common practice to expect prediction markets to satisfy some form of the following axioms.

\begin{axiom}\label{axiom:price}(Existence of Instantaneous Price): $C$ is continuous and differentiable everywhere on $\reals^{d}$.
\end{axiom}

\begin{axiom}\label{axiom:info}(Information Incorporation): for any  $q,r\in \reals^{d}$, 
$\text{Pay}(q+r,r)\geq \text{Pay}(q,r)$.
\end{axiom}

\begin{axiom}\label{axiom:noarb}(No Arbitrage) For all $q,r\in\reals^{d}$, $\exists$ $y\in\Y$ such that $\text{Pay}(q,r)\geq \langle r, \rho (y) \rangle$.
\end{axiom}

\begin{axiom}\label{axiom:express}(Expressiveness): For any $p\in\Delta_{d}$ and $\epsilon >0$, $\exists$ $q\in\reals^{d}$ such that $|| \text{InstPrice}(q)-p  || <\epsilon $.
%\rick{How do we want to show this for Smooth Market? [ ]}
\end{axiom}

\begin{axiom}\label{axiom:incentive}(Incentive Compatibility): Assume that the market is at state $q_t$ and that the agent has a belief $\mu\in\Delta_{d}$.
To maximize expected return \begin{equation}\label{eq:expecteddcfmm}
    \begin{aligned}
    \argmax_{r_t\in\reals^{d}} \quad & \underset{\text{Expected Payout}}{\underbrace{\langle \mu ,r_t \rangle}} - 
    \underset{\text{Payment to Market}}{\underbrace{\text{Pay}(q_t,r_t)}},
    \end{aligned}
    \end{equation} the agent will purchase a bundle $r_{t}$ such that for $q_{t+1}=q_t+r_t$ it holds that $\text{InstPrice}(q_{t+1})=\mu$.
\end{axiom}

\paragraph{Intuition of Axioms} 
Axiom \ref{axiom:price} states that any market state can mapped to a distribution.
Axiom \ref{axiom:info} states that if a trader were to purchase the same bundle twice, the price would be larger the second time. 
Axiom \ref{axiom:noarb} states that for any bundle purchase there exists an outcome where the trader losses money. 
Axiom \ref{axiom:express} states that the state of the market can mapped to a distribution that expresses a belief arbitrarily close.
Axiom \ref{axiom:incentive} states that a trader desires to move the market to a state such that $\nabla C$ maps the market state to their belief of the distribution of $\Y$.
See \cite{abernethy2013efficient} for further discussion regarding the intuition of axioms.
%\botodo{See abernethy et al. for more discussion}

%%%%%%%%%%%%

\subsection{\textbf{CIIP} Construction}
It may seem unclear how one could construct a $C$ which is \textbf{CIIP}.
We later show that via a $C$ being \textbf{CIIP}, a Price-Plus-Fee Market is able to automatically satisfy many of our desired axioms
In this subsection we provide supporting lemmas for constructing \textbf{CIIP} functions and provide some examples. 

\begin{lemma}[\cite{abernethy2013efficient}, Theorem 4.2, Lemma 4.3]\label{lem:ciipconst} 
Let $\hat{C}:\reals^d\to \reals$ be defined over $\relint (\Delta_d )$, $\hat{C}$ be strictly convex over its domain, and define $C=\hat{C}^*$.
As $\hat{C}$ is strictly convex, $\nabla C (q)=\argmax_{p\in \dom(\hat{C})} \langle q,p \rangle - \hat{C}(p)$. 
Furthermore, assuming $\dom (\hat{C})$ is restricted to either $\relint (\Delta_d) $ or $\Delta_d$, we also have $\text{closure}(\{\nabla C(q)\mid q\in\reals^{d} \})=\Delta_d$.
\end{lemma}
Note, if $\hat{C}:\reals^d\to \reals$ was $\frac{1}{L}$-strongly convex w.r.t. $\|\cdot \|$ then $C$ would be $L$-smooth w.r.t. $\|\cdot \|_*$ via Theorem \ref{thm:dualsmoothstr}.

A $C$ which is one-invariant allows for $\text{InstPrice}(q_t)=\text{InstPrice}(q_{t+1})$ where $q_{t+1}=q_t+\alpha \ones$ such that $\alpha\in\reals$.
This is a desirable property since the purchase of $q_{t+1}=q_t+\alpha \ones$ incorporates information uniformly. 
\begin{lemma}
    W.r.t. a function $f$ restricted to the $\Delta_d$ (or $\relint (\Delta_d)$), the Fenchel conjugate $f^*$ is one-invariant.
    %\botodo{This needs to be clarified. I think we want to say that if $f$ is restricted to the simplex, then $f^*$ is one-invariant.\rick{refers to how softmax is over relint vs sparesmax over entire simplex}}
\end{lemma}
%\bo{we should be sure to mention somewhere how CIIP relates to axioms and desirable properties. For example, here we are formally saying that ones-invariant implies that InstPrice is always a probability distribution. We should spell that out in words.}

\begin{proof} Let $q \in \reals^d$ and $\alpha \in \reals$. Observe the following $f^*(q+\alpha \ones ) = \sup_{p\in\Delta_d}\langle q+\alpha \ones,p \rangle -f(p) = \sup_{p\in\Delta_d}\langle q,p \rangle -f(p)+ \alpha\langle  \ones,p \rangle 
         = f^*(q)+\alpha$.\end{proof}

We proceed to give examples of $\hat{C}$ and the corresponding $C$ which satisfies \textbf{CIIP}.
For the case of $\hat{C} (p)=L^{-1} \sum_{i=1}^{d} p_i\log p_i $, the derived $C$ is softmax.
\paragraph{Softmax} Let $\hat{C}(p)=L^{-1}\sum_{i=1}^{d}p_i\log p_i$. 
Then $C(q)=\sup_{p\in\relint(\Delta_d)} \langle q,p \rangle - \hat{C}(p)= L^{-1}\log (\sum_{i=1}^{n}e^{L q_i})$.
Furthermore, $\nabla C(q)=\argmax_{p\in\relint (\Delta_d)} \langle q,p\rangle -\hat{C}(p)= \frac{e^{Lq}}{\sum_{i=1}^ne^{Lq_i}}$ where $C(q)$ is L-smooth w.r.t. $\|\cdot \|_2$ and $\|\cdot \|_\infty$. Also, it holds that $\{\nabla C(q)\mid q\in\reals^{d}\}=\text{int}(\Delta_d)$ and thus $\text{closure}(\{\nabla C(q)\mid q\in\reals^{d}\} ) = \Delta_d$.

For the case of $\hat{C}(p)=\frac{L}{2}\|p\|^{2}_{2}$, the derived $C$ is referred to in the machine learning literature as sparsemax \citep{martins2016softmax,niculae2017regularized}.

\paragraph{Sparsemax} Let $\hat{C}(p)=\frac{L}{2}\|p\|^{2}_{2}$ for $L>0$. 
Then $C(q)=\sup_{p\in \Delta_d} \langle q,p \rangle - \hat{C}(p)=  \sup_{p\in \Delta_d} \langle q,p \rangle -\frac{L}{2}\|p\|^{2}_{2}$.
Furthermore, $\nabla C(q)=\argmax_{p\in\Delta_d} \langle q,p\rangle -\hat{C}(p)= \argmin_{p\in\Delta_d}\| q-\frac{p}{L}\|_{2}^{2}$ where $C(q)$ is $\frac{1}{L}$-smooth w.r.t. $\|\cdot \|_2$. Also, it holds that $\{\nabla C(q)\mid q\in\reals^{d}\}= \Delta_d$ and thus $\text{closure}(\{\nabla C(q)\mid q\in\reals^{d}\} ) = \Delta_d$.

%%%%%%%%%%%%
\subsection{Duality-based CFMM}
We now define the DCFMM which provides a construction scheme for a cost function $C$ respectively and a particular $\text{Pay}(\cdot,\cdot)$ scheme. 

\begin{definition}[\cite{abernethy2013efficient}, DCFMM for AD Securities]
Let $C:\reals^{d}\to\reals$ be \textbf{CIIP} such that $C:=\hat{C}^*$ where $\hat{C}$ is strictly convex and continuous over all of $\relint (\Delta_d)$.
At state $q_t$ for bundle $r_t$, we define the payment of a DCFMM by $\text{Pay}_D(q_t,r_t)=C(q_t+r_t)-C(q_t)$.
\end{definition}

With respect to the Price-Plus-Fee Market (Definition \ref{def:lfdecomp}), one could think of the DCFMM market payment as a linear term plus an implicit fee based on the Bregman divergence by
$$\text{Pay}_{D}(q_t,r_t)=C(q_{t+1})-C(q_t)
=  \langle p_t,r_t \rangle 
+\underset{\text{Breg.-Fee}}{\underbrace{ D_C(q_{t+1}, q_t) }}$$
since $D_{C}(q_{t+1},q_{t})=C(q_{t+1})-C(q_{t})-\langle p_t,r_t \rangle$ and $q_{t+1}=q_t+r_t$.
%\bo{suggest removing the first underbrace or just making it say ``total payment'', since ``payment to DCFMM'' could sound like there are other payments. Suggest putting an underbrace on $\langle p_t,r_t \rangle$ to call it the ``price paid for the shares''.}
%\botodo{we could add a sentence or two that intuitively, the agent pays the current price for the shares, but to protect the market maker from large trades at outdated prices, we add the Bregman fee, and maybe we can try to mention that it is a ``curvature fee''.}
Intuitively, the DCFMM charges the instantaneous price per share in a bundle $r_t$ via $\langle p_t,r_t\rangle$ and then charges a a ``curvature fee'' via the Bregman Fee to hedge the potential loss for large trades.
The DCFMM satisfies Axioms 1-5, which we now show. 
%\footnote{Worth noting, the DCFMM was also shown in \cite{abernethy2013efficient} to satisfy an axiom called Path Independence; however, we do not address said axiom in this work.}

\begin{theorem}[\cite{abernethy2013efficient}, Theorem 3.2]
The DCFMM satisfies Axioms 1-4.
\end{theorem}
Trivially, with $\text{Pay}_D$ plugged into Eq (\ref{eq:expecteddcfmm}) and taking the gradient with respect to $q_{t+1}$ and setting it equal to zero, we can see that Axiom \ref{axiom:incentive} is also satisfied for the DCFMM.
Furthermore, the DCFMM has the nice property that the worst-case loss of the market is bounded.
\begin{theorem}[\cite{abernethy2013efficient}, Theorem 4.4]
The DCFMM has a worst-case loss no more than $\sup_{p\in\rho(\Y )}R(p)-\min_{p\in\Delta_d}R(p)$.
\end{theorem}

%%%%%%%%%%%%
As noted by \cite{abernethy2013efficient}, FTRL's regrate rate and DCFMM's worst-case loss are similar.
We show in Appendix \ref{app:cftrlequiv} the equivalence to Continuous FTRL's regret rate and DCFMM's worst-case loss.

%% file: quad_market.tex
\section{Smooth Quadratic Prediction Markets}

%%%%%%%%%%%%%%%%%%%%%%%%%%%%%%%%%%%%%%%%%%
%%%%%%%%%%%%%%%%%%%%%%%%%%%%%%%%%%%%%%%%%%

%%%%%%%%%%%%%%%%%%%%%%%%%%%%%%%%%%%%%%%%%%
%%%%%%%%%%%%%%%%%%%%%%%%%%%%%%%%%%%%%%%%%%

%%%%%%%%%%%%%%%
%%%%%%%%%%%%%%%

\subsection{Smooth Quadratic Prediction Market Design}
We now introduce the Smooth Quadratic Prediction Market.
Given a smooth \textbf{CIIP} function $C$ w.r.t. a general norm $\|\cdot \|$, we propose charging a fee based on the upper quadratic bound obtained via the $L$-smoothness condition.

\begin{definition}[Smooth Quadratic Prediction Market]
Let $C:\reals^{d}\to\reals$ be \textbf{CIIP}.
Assume that $C$ is L-smooth w.r.t. $\|\cdot \|$. 
At state $q_t$ for bundle $r_t$, we define the payment of a Smooth Quadratic Prediction Market by 
$$ \text{Pay}_L(q_t,r_t) =\langle p_t,r_t \rangle + \underset{\text{Q-Fee}}{\underbrace{ \frac{L}{2}\| r_t\|^{2} }} ~.$$
\end{definition}
In the remainder of this section, we show that the Smooth Quadratic Prediction Market satisfies Axioms 1-4 and has a better worst-case loss than the DCFMM.
Note, we don't claim that the Smooth Quadratic Prediction Market satisfies Axiom \ref{axiom:incentive} Incentive Compatibility.
Later in Section \ref{subsec:el2} and \ref{subsec:elp}, we show that the traders within the Smooth Quadratic Prediction Market satisfy instead a form of \textit{incremental} incentive compatibility, and interestingly, the traders mimic the update steps of gradient (general) steepest descent while trading.

\begin{lemma}
    By the assumption that $C$ is \textbf{CIIP}, the Smooth Quadratic Prediction Market satisfies Axiom \ref{axiom:price} Existence of Instantaneous Price. 
\end{lemma}

\begin{lemma}
    The Smooth Quadratic Prediction Market satisfies Axiom \ref{axiom:info} Information Incorporation. 
\end{lemma}
\begin{proof}
     Since $C$ is convex and differentiable, by the monotonicity of gradients (Lemma \ref{lemma:monograd}), we have that 
        $$\langle  \nabla C(q_{t+1})-\nabla C(q_t) ,q_{t+1}-q_t \rangle  \geq 0   \Leftrightarrow \langle  \nabla C(q_{t}+r_t), r_t\rangle   \geq \langle \nabla C(q_t) , r_t \rangle $$
        where $q_{t+1}=q_t+r_t$.
        Then by adding $\frac{L}{2}\|r_t\|^{2}$ to both sides we get 
        $$\langle  \nabla C(q_{t}+r_t), r_t\rangle +\frac{L}{2}\|r_t\|^{2}  \geq \langle \nabla C(q_t) , r_t \rangle +\frac{L}{2}\|r_t\|^{2}  \Leftrightarrow \text{Pay}_L(q_t+r_t,r_t)  \geq \text{Pay}_L(q_t,r_t) ~.$$
\end{proof}

\begin{lemma}
    The Smooth Quadratic Prediction Market satisfies Axiom \ref{axiom:noarb} No Arbitrage. 
\end{lemma}
\begin{proof}
    By (Theorem 3.2, \cite{abernethy2013efficient}) the DCFMM satisfies no arbitrage, i.e, $C(q_t+r_t)-C(q_t)\geq \langle r_t,  \rho (y)  \rangle$ for some $y\in\Y $.
    However, it also holds that $$C(q_t+r_t)-C(q_t)= \langle p_t,r_t \rangle 
+ D_C(q_{t+1}, q_t)  \leq \langle p_t,r_t \rangle 
+ \frac{L}{2}\| r_t \|^2$$ via Definition \ref{def:bregsmoothstrong}. 
Therefore by combing the two inequalities we have that $ \langle r_t, \rho (y) \rangle \leq \langle p_t,r_t \rangle +\frac{L}{2}\| r_t\|^{2}$.
\end{proof}

\begin{lemma}
    The Smooth Quadratic Prediction Market satisfies Axiom \ref{axiom:express} Expressiveness. 
\end{lemma}
\begin{proof}
By the \textbf{CIIP} assumption on $C$, it holds that $\text{closure}(\{\nabla C(q)\mid q\in\reals^{d}\} ) = \Delta_d$ which states every limit point of the set $\{\nabla C(q)\mid q\in\reals^{d}\}$ is in $\Delta_d$.
Hence, by definition of Expressiveness, the claim holds. 
\end{proof}

\begin{theorem}
    For any fixed trade history $h=(r_0,\dots ,r_t)$, the Smooth Quadratic Prediction Market has a better worst-case loss then the DCFMM.
\end{theorem}
\begin{proof}
    Observe for any $(q_i,r_i)$ for $i\in 0\dots ,t$ that 
    $$C(q_i+r_i)-C(q_i)=\langle p_i,r_i \rangle 
    + D_C(q_{i+1}, q_i)  \leq  \langle p_i,r_i \rangle 
    + \frac{L}{2}\| r_i \|^2$$ via Definition \ref{def:bregsmoothstrong}. 
    Hence, the collected revenue of the the Smooth Quadratic Prediction Market is greater overall than that of the DCFMM implying a better worst-case loss then the DCFMM.  
\end{proof}

%%%%%%%%%%%%%%%%%%%%
%%%%%%%%%%%%%%%%%%%%

%%%%%%%%%%%%%%%%%%%%
%%%%%%%%%%%%%%%%%%%%

\subsection{$\ell_2$-based Smooth Quadratic Prediction Markets}\label{subsec:el2}
%\botodo{To consider: discuss how we're decoupling the payment and the learning algorithm, to an extent. The market maker is still using FTRL with e.g. entropy regularizer. But the agents perform steepest descent, which is kind of weird...}
As mentioned, the core purpose of a prediction market with AD securities is to elicit a distribution over a finite set of outcomes.
Hence, it is of essence to show that traders are incentivized to move the market state such that $\nabla C$ maps the market state to a trader's belief; however, recall that we stated the Smooth Quadratic Prediction Market does not satisfy the Axiom of Incentive Compatibility.
Interestingly, the expectation maximizing trading behavior of an agent w.r.t. the $\ell_2$-based Smooth Quadratic Prediction Market is expressible via gradient descent.
We show that the $\ell_2$-based Smooth Quadratic Prediction Market satisfies a form of \textit{incremental} incentive compatibility.
Overall, the core result of this section, Theorem \ref{thm:incentivel2}, states that a trader is incentivized to move the market state such that the market state maps to their belief via a sequence of bundle purchases instead of via a single transaction. 
We formally define incremental incentive compatibility as follows. 

%\botodo{To consider: put this with the other axioms.}
\begin{axiom}\label{ax:incrementalincentive}(Incremental Incentive Compatibility): Assume the market is at state $q_0$ and that a sequence of agents with the same belief $\mu\in\Delta_{d}$ purchases bundles $r_t$ relative to maximizing their expected payout
\begin{equation*}\label{eq:expecteddcfmmincre}
    \begin{aligned}
\argmax_{r_t\in\reals^{d}} \quad & \underset{\text{Expected Payout}}{\underbrace{\langle \mu ,r_t \rangle}} - 
    %\argmax_{r_t \in\reals^{d}} \quad & \underset{\text{Expected Payout}}{\underbrace{\langle \mu ,r_t \rangle}} - 
    \underset{\text{Payment to Market}}{\underbrace{\text{Pay}(q_t,r_t)}}.
    \end{aligned}
    \end{equation*}
%\bo{Suggest making the argmax over $r_t$ (since that's what's in the expression) and eliminating this reference to $q_{t+1}$ altogether, since it doesn't seem needed here.}
Then $\lim_{t\to\infty}\nabla C(q_t)= \mu$.
\end{axiom}

%\begin{definition}[$\ell_2$-based L-smoothness]
% We say that a function $f$ is $L$-smooth, for some $L \geq 0$, iff $\| \nabla f(x)-\nabla f(y) \|_2\leq  L \|x-y\|_2$, for all x and y.
%\end{definition}

%\bo{It's not clear in the narrative what Def 8, Thm 5, and Lemma 9 are doing here. I suggest moving Thm 5 and Lemma 9 to the appendix unless we have a point we want to make with them, and if so, we should discuss the point. I think we can keep Def 8 and then, in the statement of Theorem 6, we can say that the updates made by the agents follow gradient descent with step size $L$.}
We now formally define gradient descent (GD) and provide a supporting Theorem and Lemma for GD which will be used in proving incremental incentive compatibility for the $\ell_2$-based Smooth Quadratic Prediction Market. 

\begin{definition}[Gradient Descent]
    Let $x_0\in\reals^d$, and let $\gamma >0$ be a step size. 
    Given a differentiable function $f$, the gradient descent (GD) algorithm defines a sequences $(x_t)_{t\in\N_0}$ satisfying $x_{t+1}=x_t-\gamma \nabla f(x_t).$
\end{definition}

\begin{comment}
 \begin{theorem}[\cite{bauschke2017convex}, Corr 28.9]\label{thm:stateconv}
Assume that $f$ is convex and $\ell_2$-based $L$-smooth with respect to a norm, for some $L > 0$. 
Let $(x_t)_{t\in\N}$ be the sequence of iterates generated by the (GD) algorithm, with a stepsize satisfying $0<\gamma\leq\frac{2}{L}$. 
Then $(x_t)_{t\in\N}$ converges weakly to a point in $\argmin (f)$.
\end{theorem}

\begin{definition}
Let $X$ be a normed linear space, and let $x_n, x \in X$.
We say that $x_n$ converges weakly to $x$, if for all $\mu\in X^*$, $\lim_{n\to\infty}\langle x_n,\mu\rangle=\langle x,\mu\rangle$
\end{definition}   
\end{comment}

\begin{theorem}[\cite{garrigos2023handbook}, Theorem 3.4]\label{thm:gd2}
Assume that $f$ is convex and $\ell_2$-based $L$-smooth, for some $L > 0$.
Let $(x_t)_{t\in\N_0}$ be the sequence of iterates generated by the GD algorithm, with a stepsize satisfying $0<\gamma\leq\frac{1}{L}$.
Then, for all $x^*\in\argmin f$ and for all $t\in\N_0$ we have that $$ f(x^{(t)})-\inf f\leq \frac{\|x^{(0)}-x^*\|^2_{2}}{2\gamma t} ~.$$
\end{theorem}

\begin{lemma}[\cite{sidford2024}, Lemma 6.1.6]\label{sid616}
    If $f:\reals^d\to \reals$ is L-smooth w.r.t. $\|\cdot \|$ then for all $x^*\in \argmin f$ and $x\in\reals^{d}$ it holds that 
    $$ \frac{1}{2L}\|\nabla f(x)\|^{2}_{*}\leq f(x)-f(x^*)\leq \frac{L}{2}\|x-x^*\|^{2}~.$$
\end{lemma}

\begin{theorem}\label{thm:incentivel2}
    %Let $C:\reals^{d}\to\reals$ be \textbf{CIIP} and $\ell_2$-based L-smooth. 
    %Say the market has a current state of  $q_0\in\reals^{d}$. \bo{I think this sentence can be deleted, since this should be assumed now}
    With respect to some \textbf{CIIP} $C$, define a $\ell_2$-based Smooth Quadratic Prediction Market.
    The market satisfies Axiom \ref{ax:incrementalincentive} Incremental Incentive Compatibility furthermore $\lim_{t\to\infty}\nabla C(q_t)= \mu$ at a rate of $\frac{1}{t}$. 
    %\bo{instead of ``such that'', suggest ``and furthermore''}
    %Assume a sequence of agents with the same belief $\mu\in\Delta_{d}$ purchases bundles relative to maximizing their expected payout
    %\begin{equation}
    %\begin{aligned}
    %\max_{q_{t+1}\in\reals^{n}} \quad & \underset{\text{Expected Payout}}{\underbrace{\langle \mu ,q_{t+1}-q_t \rangle}} - 
    %\underset{\text{Payment to Market}}{\underbrace{\Big( \langle \nabla C(q_t),q_{t+1}-q_t \rangle +\frac{L}{2}\| q_{t+1}-q_t  \|^{2}_2  \Big)}}.
    %\end{aligned}
    %\end{equation}
    %Then $\lim_{t\to\infty}\nabla C(q_t)= \mu$ at a rate of $\frac{1}{t}$.
\end{theorem}

\begin{proof}
Let $\overline{C}(q)=C(q)-\langle \mu,q\rangle$ and note that $\nabla \overline{C}(q)= \nabla C(q)-\mu $.
Note that the utility of an agent is equivalent to the following
\begin{equation*}
    \begin{aligned}
    \min_{q_{t+1}\in\reals^{d}} \quad & \underset{\text{Payment to Market}}{\underbrace{\Big( \langle \nabla C(q_t),q_{t+1}-q_t \rangle +\frac{L}{2}\| q_{t+1}-q_t  \|^{2}_2  \Big)}}-\underset{\text{Expected Payout}}{\underbrace{\langle \mu ,q_{t+1}-q_t \rangle}}  
    .
    \end{aligned}
    \end{equation*}
Hence, the update of the market state $q_{t+1}\leftarrow q_t-(\frac{1}{L})( \nabla C(q_t)-\mu)= q_t-(\frac{1}{L})\nabla\overline{C}(q_t)$ is a GD step performed on $\overline{C}$.
Observe, setting $\nabla\overline{C}(q)=0$ we get that $\nabla C(q)=\mu$, i.e. for $q^*\in \argmin \overline{C}(q)$ it holds that $\nabla C(q^*)=\mu$.  
Hence,  $q_{t+1}\leftarrow q_t-(\frac{1}{L})( \nabla C(q_t)-\mu)= q_t-(\frac{1}{L})\nabla\overline{C}(q_t)$ is a GD update for $\overline{C}(q)=C(q)-\langle \mu,q\rangle$ with learning rate $\frac{1}{L}$ whose minimum point is a $q^*$ such that $\nabla C(q^*)=\mu$.
Also note that if $\nabla C(q_t)=\mu$ then $q_{t+1}\leftarrow q_t-(\frac{1}{L})( \nabla C(q_t)-\mu)= q_t-(\frac{1}{L})\nabla\overline{C}(q_t)$ hence, any minimizer of $\overline{C}$ is a stationary point. 
Observe that $\overline{C}$ is $\ell_2$-based smooth and hence Theorem \ref{thm:gd2} and Lemma \ref{sid616} apply.
\end{proof}

%%%%%%%%%%%%%%%%%%%%%%%%%%%%%%%%%

%%%%%%%%%%%%%%%%%%%%%%%%%%%%%%%%%
%%%%%%%%%%%%%%%%%%%%
%%%%%%%%%%%%%%%%%%%%

\subsection{$\ell_{p}$-based Smooth Quadratic Prediction Markets}\label{subsec:elp}

We now generalize our results for the Smooth Quadratic Prediction Market defined w.r.t. $\|\cdot \|_p$ for $p\in [1,\infty ]$.
By altering the norm, the Smooth Quadratic Prediction Market gets varying behavior pertaining to the charge for Q-fee and path of convergence of $\{q_t\}_{t\in\mathbb{N}_0}\to \argmin_{q\in\reals^{d}}\overline{C}(q)$.
In this generalized setting, agents are incentivized to trade following the general steepest descent algorithm in order to maximize their expected return.
We now formally define general steepest descent (SD) and provide a supporting Theorem for SD which will be used in proving incremental incentive compatibility for the $\ell_p$-based Smooth Quadratic Prediction Market. 

%\rick{Working on matching quad and general steepest descent framing, refer to \url{https://arxiv.org/pdf/2002.08056}}
%Steepest descent updates via $x_{t+1}=x_t +\gamma \cdot \nabla^{\|\cdot \|} f(x_t)$, where $\nabla^{\|\cdot \|} f(x_t)= \|u \|\cdot u$ and $u\in \argmin_{ \|v  \|_{*}\leq 1} \langle \nabla f(x_t),v \rangle $. 
%For the case where $\|\cdot \|_2$, steepest descent reduces to gradient descent

\begin{definition}[General Steepest
Descent]
 Given a function $f$ that is differentiable and L-smooth w.r.t. $\|\cdot \|$, the general steepest descent (SD) method w.r.t. the norm $\|\cdot \|$ iteratively defines a sequences $(x_t)_{t\in\N_0}$ via $x_{t+1}\in \argmin_{x\in\reals^{d}} \langle \nabla f(x_t),x-x_t\rangle+\frac{L}{2}\|x-x_t\|^2$.
%\begin{equation*}\label{eq:steepdescent}
%x_{t+1}\in \argmin_{x\in\reals^{d}} \langle \nabla f(x_t),x-x_t\rangle+\frac{L}{2}\|x-x_t\|^2.    
%\end{equation*}

\end{definition}

%\bo{Again, I think we should either explain what we're doing with this theorem, or just move it to the appendix.}
\begin{theorem}[\cite{kelner2014almost}, Theorem 1]\label{thm:gensteept}
If $f$ is convex and L-smooth w.r.t. $\|\cdot \|$, then SD w.r.t. $\|\cdot \|$ satisfies $f(x_t)- \inf f \leq \frac{2LK^2}{t+4} \text{ with } K:=\max_{x:f(x)\leq f(x_0)}\min_{x^*:f(x^*)=\inf f}\|x-x^*\|$
where $(x_t)_{t\in\N_0}$ is a sequence of iterates generated by the SD algorithm.
\end{theorem}

%%%%%%%%%%%%%%%%%%%%%%%%%%%%%%%%%
%%%%%%%%%%%%%%%%%%%%%%%%%%%%%%%%%

%%%%%%%%%%%%%%%%%
%%%%%%%%%%%%%%%%%
\begin{figure}[t]
	\centering
	\begin{subfigure}{0.32\linewidth}
		\includegraphics[width=\linewidth]{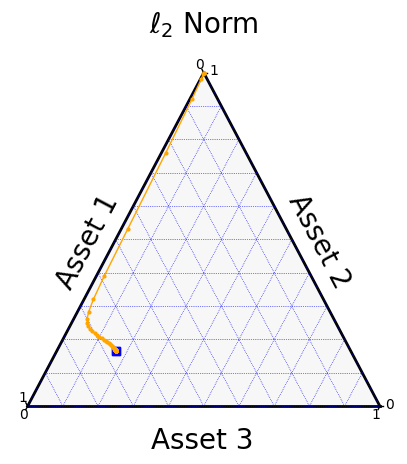}
		%\caption{Yellow color}
		\label{fig:subfigA}
	\end{subfigure}
	\begin{subfigure}{0.32\linewidth}
		\includegraphics[width=\linewidth]{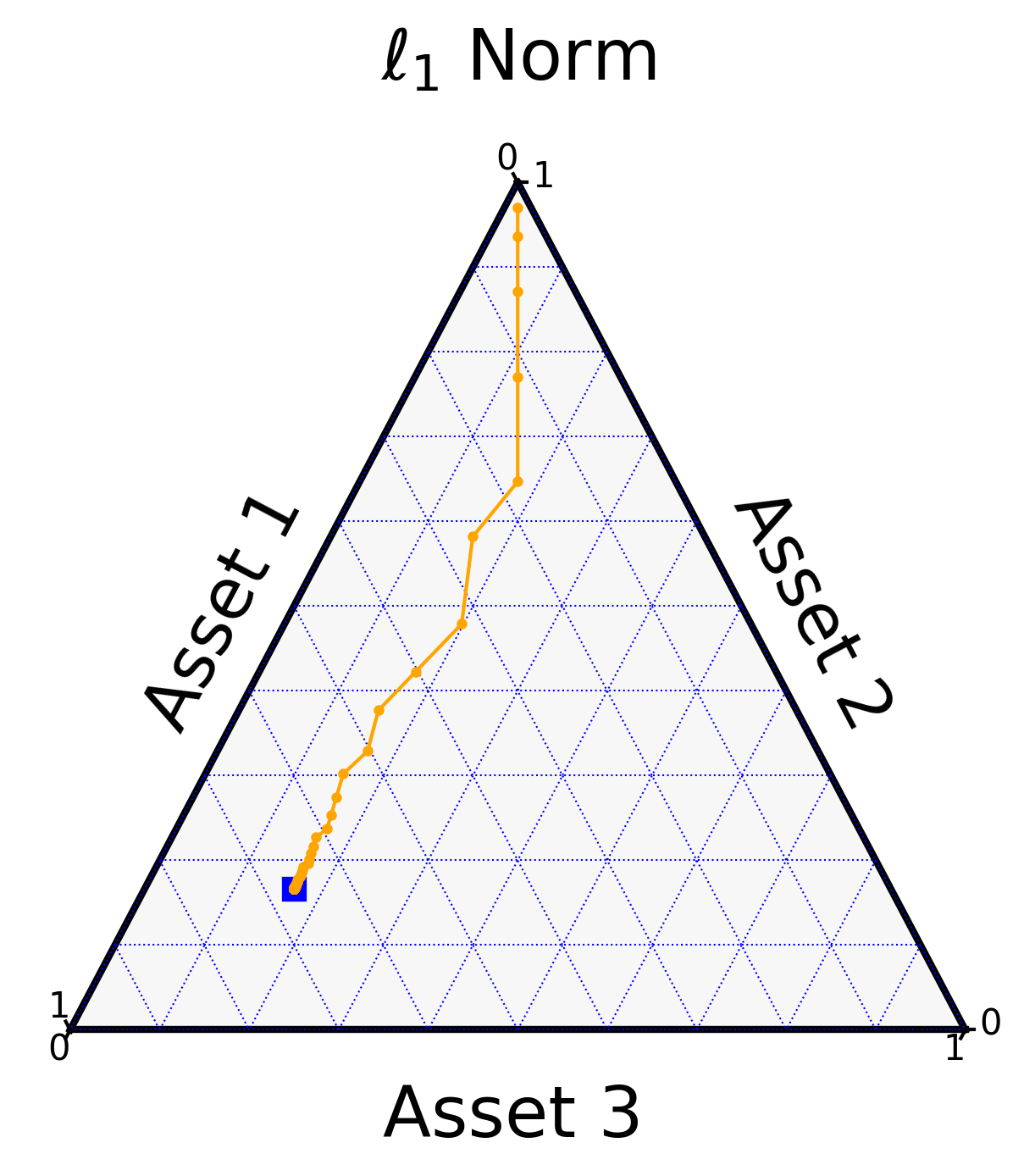}
		%\caption{Red color}
		\label{fig:subfigB}
	\end{subfigure}
	\begin{subfigure}{0.32\linewidth}
	        \includegraphics[width=\linewidth]{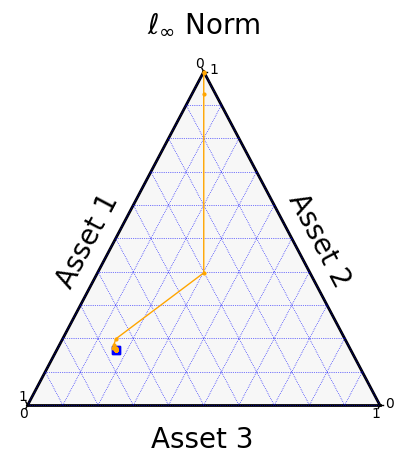}
	        %\caption{Green color}
	        \label{fig:subfigC}
         \end{subfigure}
	\caption{Let $q_{0}=(10,20,10)$, $C$ is softmax with smoothness of $L=1$, and $\mu = (1/6,1/6,2/3)$. The blue square expresses $\mu$ and the orange path towards the blue square demonstrates the updating market distribution states. As denote by the titles's of each plot, we vary the norm used for the Smooth Quadratic Prediction Market. Note although softmax is not $\ell_{1}$-smooth, we use said norm experimentally for the sake of comparison.}\label{fig:subfigures}
\end{figure}
%%%%%%%%%%%%%%%%
%%%%%%%%%%%%%%%%

\begin{theorem}\label{thm:genconv}
    %Let $C:\reals^{d}\to\reals$ be \textbf{CIIP} and $\ell_p$-based L-smooth for $p\in [1,\infty]$. 
    %Say the market has a current state of  $q_0\in\reals^{d}$.
    With respect to some \textbf{CIIP} $C$, define an $\ell_p$-based Smooth Quadratic Prediction Market.
    The market satisfies Axiom \ref{ax:incrementalincentive} Incremental Incentive Compatibility furthermore $\lim_{t\to\infty}\nabla C(q_t)= \mu$ at a rate of $\frac{1}{t}$. 
    %\bo{instead of ``such that'', suggest ``and furthermore''}
    %Assume a sequence of agents with the same belief $\mu\in\Delta
    %_{d}$ purchases bundles relative to maximizing their expected payout
    %\begin{equation}\label{eq:smoothpquadobj}
    %\begin{aligned}
    %\max_{q_{t+1}\in\reals^{n}} \quad & \underset{\text{Expected Payout}}{\underbrace{\langle \mu ,q_{t+1}-q_t \rangle}} - 
    %\underset{\text{Payment to Market}}{\underbrace{\Big( \langle \nabla C(q_t),q_{t+1}-q_t \rangle +\frac{L}{2}\| q_{t+1}-q_t  \|^{2}_p  \Big)}}.
    %\end{aligned}
    %\end{equation}
    %or expressed equivalently as 
    %\begin{equation}\label{eq:minpsmooth}
    %\begin{aligned}
    %\min_{q_{t+1}\in\reals^{n}} \quad & 
    %\underset{\text{Payment to Market}}{\underbrace{\Big( \langle \nabla C(q_t),q_{t+1}-q_t \rangle +\frac{L}{2}\| q_{t+1}-q_t  \|^{2}_p  \Big)}}-\underset{\text{Expected Payout}}{\underbrace{\langle \mu ,q_{t+1}-q_t \rangle}}.  
    %\end{aligned}
    %\end{equation}
    %Then $\lim_{t\to\infty}\nabla C(q_t)= \mu$ at a rate of $\frac{1}{t}$.
\end{theorem}

\begin{proof}
Observe that the utility of an agent is equivalent to the following \begin{equation}\label{eq:surrpnorm}
    \begin{aligned}
    \argmin_{q_{t+1}\in\reals^{d}} \quad & 
    \underset{\text{Payment to Market}}{\underbrace{\Big( \langle \nabla C(q_t),q_{t+1}-q_t \rangle +\frac{L}{2}\| q_{t+1}-q_t  \|^{2}_p  \Big)}}-\underset{\text{Expected Payout}}{\underbrace{\langle \mu ,q_{t+1}-q_t \rangle}} \\
    \Leftrightarrow \argmin_{q_{t+1}\in\reals^{d}} \quad & 
     \langle \nabla \overline{C}(q_t),q_{t+1}-q_t \rangle +\frac{L}{2}\| q_{t+1}-q_t  \|^{2}_p 
    \end{aligned}
    \end{equation}
    as $q_{t+1}=q_t+r_t$.
    Hence, Eq. (\ref{eq:surrpnorm}) is a SD step for $\overline{C}=C(q)-\langle \mu,q\rangle$.
    Observe that $\overline{C}$ is $\ell_p$-based smooth and hence Theorem \ref{thm:gensteept} and Lemma \ref{sid616} apply.
\end{proof}

In Figure \ref{fig:subfigures}, we demonstrate the varying convergence behavior of the Smooth Quadratic Prediction Market when using different norms.
There are numerous convergence results for gradient (general steepest) descent.
To avoid redundancy with the literature, we refer the reader to Appendix \ref{app:steep} for further convergence results, which may be of interest to the application of Smooth Quadratic Prediction Markets.

%%%%%%%%%%%%%%%%%%%%%%%%%%%%%%

%\begin{lemma}[\cite{rockafellar1997convex}, Corr. 27.2.1] Let $f$ be a closed proper convex function which has no direction of recession. Let $x_1,x_2,\dots$ be any sequence such that $\lim_{i\to\infty}f(x_i)=\inf f$. Then $x_1,x_2,\dots$ is a bounded sequence, and all its cluster points belong to the minimum set of $f$.
%\end{lemma}

%\rick{Direction of recession: \url{https://tisp.indigits.com/cvxopt/recession_opt#:~:text=A%20direction%20of%20recession%20of%20a%20proper,never%20increases%20in%20a%20direction%20of%20recession.}}

%%%%%%%%%%%%%%%%%%
%%%%%%%%%%%%%%%%%%

%%%%%%%%%%%%%%%%%%%%%%%%%%%%%
%%%%%%%%%%%%%%%%%%%%%%%%%%%%%

%% file: constraints.tex
\section{Smooth Quadratic Prediction Markets with Constraints}
In this section, we examine the trading behavior of expectation-maximizing agents under two realistic constraints (independently): budget-bounded traders and buy-only markets.
For mathematical sake, we solely examine the $\ell_2$-based Smooth Quadratic Prediction Market buy-only market and perform experiments for the budget-bounded trader setting.
Furthermore, for both types of constraints, we perform simulations for the $\ell_1$ and $\ell_{\infty}$ cases.
Overall, either analytically or experimentally, we observe that in both situations, the market state converges to the belief of a sequence of traders, regardless of the used p-norm, motivating the use of the Smooth Quadratic Prediction Market in either of these realistic scenarios.%\footnote{Code can be found at: https://github.com/EnriqueNueve/Smooth-Quadratic-Prediction-Markets}
\footnote{Code can be found at: https://github.com/EnriqueNueve/Smooth-Quadratic-Prediction-Markets}
We refer the reader to Appendix \ref{app:cons} for the omitted analysis within this section. 

\paragraph{Budget-Bounded Traders} 

\begin{comment}
From solving the Lagrange dual problem, it is not clear if the market does converge to an agent's belief given a sufficiently small budget.
In more detail, the market converges up to the belief adjusted by a small normalization factor that depends on the budget which adapts over time.
However, experimentally, for budgets as low as $B=.1e-7$, we still achieved convergence. Hence, the impact of this technical nuance regarding convergence to the belief with a bounded budget is minimal. 
The trader's objective with a budget $B\in\reals_{>0}$ is shown below.

\begin{equation*}
\begin{aligned}
\min_{q_{t+1}\in\reals^{d}} \quad &  
\underset{\text{Payment to Market}}{\underbrace{\Big( \langle \nabla C(q_t),q_{t+1}-q_t \rangle +\frac{L}{2}\| q_{t+1}-q_t  \|^{2}_{2}  \Big)}} - \underset{\text{Expected Payout}}{\underbrace{\langle \mu ,q_{t+1}-q_t \rangle}} \\
\textrm{s.t.} \quad & \underset{\text{Payment to Market}}{\underbrace{\Big( \langle \nabla C(q_t),q_{t+1}-q_t \rangle +\frac{L}{2}\| q_{t+1}-q_t  \|^{2}  \Big)}} \leq \underset{\text{Budget}}{\underbrace{B}} \\
\end{aligned}
\end{equation*}
where $B\in \reals_{>0}$.
\end{comment}

Following the work of \cite{fortnow2012multi}, we consider the notation of \textit{natural budget constraint}, which states that the loss of the agent is at most their budget, for all $y\in\Y$. 
The objective of an agent is expressed by 
\begin{equation*}
\begin{aligned}
\min_{r_t \in\reals^{d}} \quad &  
\underset{\text{Payment to Market}}{\underbrace{\Big( \langle \nabla C(q_t),r_t \rangle +\frac{L}{2}\| r_t  \|^{2}_{2}  \Big)}} - \underset{\text{Expected Payout}}{\underbrace{\langle \mu ,r_t \rangle}} \\
\textrm{s.t.} \quad & \underset{\text{Payment to Market}}{\underbrace{\Big( \langle \nabla C(q_t),r_t \rangle +\frac{L}{2}\| r_t  \|^{2}_{2}  \Big)}}-\underset{\text{Realized Payout}}{\underbrace{\langle \delta_y ,r_t \rangle}} \leq \underset{\text{Budget}}{\underbrace{B}}, \quad \forall \; y\in\Y  \\
\end{aligned}
\end{equation*}
where $B\in \reals_{>0}$.
Experimentally, using $\nabla C$ equal to softmax, we were able to observe convergence regardless of the initial state, belief distribution, and budget used.
The rate of convergence was proportional to the size of the budget. 
Figure \ref{fig:subfiguresbudget} in Appendix \ref{app:cons} demonstrates the convergence behavior of the market under budget constraints.

\paragraph{Buy-Only Market} The Buy-Only Market \cite{li2013axiomatic} assumes that $0 \preceq r_t$ for all $t\in\mathbb{N}_0$; hence, only positive bundles purchases are allowed.
The objective of the trader is expressed via the equation below.
\begin{equation*}
\begin{aligned}
\min_{r_t \in\reals^{d}} \quad &  
\underset{\text{Payment to Market}}{\underbrace{\Big( \langle \nabla C(q_t),r_t \rangle +\frac{L}{2}\| r_t  \|^{2}_{2}  \Big)}} - \underset{\text{Expected Payout}}{\underbrace{\langle \mu ,r_t \rangle}} \\
\textrm{s.t.} \quad & \underset{\text{Buy Only Constraint}}{\underbrace{0  \preceq r_t}}
 \\
\end{aligned}
\end{equation*}

From solving the Lagrange dual problem, we get that the update step is
$q_{t+1}=q_t+\frac{1}{L}( (\nabla C(q_t)-\mu )_+- (\nabla C(q_t)-\mu ))$
and observe when $\nabla C(q_t)=\mu$ we have a stationary point. 
Observe the update is a coordinate GD update for $\overline{C}$ when $(\nabla C(q_t)-\mu )_i <0$ and $q_{t+1,i}=q_{t,i}$ when $(\nabla C(q_t)-\mu )_i \geq 0$. 
Hence, not just experimentally but analytically, we are able to show incremental incentive compatibility. 

%% file: adaptive_liquid_new.tex
\section{Smooth Quadratic Prediction Markets with Adaptive Liquidity }
Inspired by the paper \textit{A General Volume-Parameterized Market Making Framework} \citep{abernethy2014general}, we adapt aspects of the framework to our setting to facilitate adaptive liquidity.
Liquidity in this context can be thought of as the sensitivity of price changes with respect to bundle purchases. 
A high liquidity would indicate that a large bundle purchase would not cause a significant price change and vice-versa. 
Adaptive liquidity can be facilitated by decreasing the smoothness of $C$ as the volume of trades increases.

Let $\mathbf{S} = (\reals^d )^*$ denote the history space of the market consisting of finite (and possibly empty) sequences of bundles.
We define a volume update function via an asymmetric norm. 
%The markets we consider will be defined using a cost function $N : S \times S \to \reals$, where $N(s; s')$ is the cost of purchasing the sequence $s \in S$ of bundles given the current history $s'\in S$

%%%%%%%%%%%%%%%
\begin{comment}
  \begin{definition}[Volume]
The function $V : \mathbf{S} \to \reals_{+}$ is a volume function if for all $s\in\mathbf{S}$ and $r\in\reals^d$ it satisfies
\begin{enumerate}
    \item $V(s\oplus r)\geq V(s)$
    \item  $V(s\oplus s')$ is unbounded in $|s'|$ where $s'=r\oplus r\oplus \dots \oplus r$ for some $r\neq 0$
    \item For all $s'\in S$, $V(s\oplus (\alpha r)\oplus s')$ is unbounded in $\alpha$ for $r\neq 0$. 
\end{enumerate}
\end{definition}  
\end{comment}
%%%%%%%%%%%%%%%

\begin{definition}[VPM]
  Say we have a finite random variable $\Y =\{y_1, \dots ,y_d\}$ over $d$ mutually exclusive and exhaustive outcomes.
  Let $C^{\circ}:\reals^{d}\times\reals_{+}\to\reals$ be \textbf{CIIP} and increasing in $v$ (it's second argument).
  Assume $C^{\circ}$ is L-smooth w.r.t. some general norm $\|\cdot \|$.
  Furthermore, as $v$ increases, the smoothness of $C^{\circ}$ should decrease.
  Let $g:\reals^{d}\to\reals_+$ be an asymmetric norm (refer to Def. \ref{def:asymnorm} in Appendix ~\ref{app:liquid}) and define w.r.t. a $s=r_0\dots,r_t \in \mathbf{S}$ the volume update function $V(s)=v_0+\sum_{i=0}^{t}g(r_i)$.
  The VPM for AD securities defined by $C^{\circ}$, with initial state $q_0$, and initial volume $v_0\in\reals_+$ operates as follows.  
   At round $t\in \mathbb{N}_0 $,
    \begin{enumerate}
       \item A trader can request any bundle of securities $r_t\in\reals^{d}$.
       \item The trader pays the market maker some amount $\text{Pay}(q_t,r_t;v_t)\in\reals$ (such that $\text{Pay}$ is dependent on $C^{\circ}$) in cash.
       \item The market state updates to $q_{t+1}=q_{t}+r_t$ and $v_{t+1}=V(s)$ such that $s=r_0,\dots ,r_t$.
   \end{enumerate}
After an outcome of the form $Y=y_i$ occurs, for each round $t$, the trader responsible for the trade $r_t$ is paid $(r_t)_i$ in cash, i.e. the number of shares purchased in outcome $y_i$.
The market payout for the bundle $r_t$ and the outcome $Y=y$ is expressed via $\langle r_t,\rho (y) \rangle$ where $\rho :\Y \to \delta_{y}$. 
At any state $(q_t,v_t)$, the market maker can infer the belief of the market via the instsantaneous price $\text{InstPrice}(q_t;v_t) = f(q_t,0;v_t)$ where $f(q_t,r_t;v_t) = \nabla_{r_t} \text{Pay}(q_t,r_t;v_t)$.
\end{definition}

For the case of the VPM-DCFMM, the pay function is the following $$\text{Pay}_{D^{\circ}}(q_t,r_t;v_t)=C^{\circ}(q_t+r_t;v_t+g(r_t))-C^{\circ}(q_t;v_t)~.$$ 

Motivated by an inequality due the smoothness of $C^\circ$ (shown in Appendix \ref{app:liquid}), we propose charging 
$$ \text{Pay}_{L^{\circ}}(q_t,r_t;v_t)= \underset{\text{Payment to Smooth-Quad Market}}{\underbrace{\Big( \langle \nabla C^{\circ}(q_t;v_t+g(r_t)),r_t \rangle +\frac{L^{\circ}}{2}\|r_t\|^{2}  \Big)}}+
\underset{\text{Liquidity Fee}\geq 0}{\underbrace{\Big( C^{\circ}(q_{t};v_t+g(r_t))-C^{\circ}(q_t;v_t)\Big)}}~.$$
Using this approach, we can easily show that $\text{Pay}_{L^{\circ}}(q_t,r_t;v_t)$ facilitates no arbitrage (proof in Appendix \ref{app:liquid}). 

\begin{restatable}{lemmac}{adaparb}
  \label{lem:adapliquid}
    $\text{Pay}_{L^{\circ}}$ satisfies Axiom \ref{axiom:noarb} No Arbitrage.
\end{restatable}

%\begin{proof}
%    By (\cite{abernethy2014general}, Lemma 3.6) the VPM-DCFMM satisfies no arbitrage, i.e, for all $r_0,\dots ,r_t\in\textbf{S}$ it holds that $\langle \rho (y),\sum_{i=0}^{t}r_i \rangle \leq \sum_{i=0}^{t} \text{Pay}_{D^{\circ}}(q_i,r_i;v_i)$ for some $y\in\Y $.
%    However, it also holds that $\sum_{i=0}^{t} \text{Pay}_{D^{\circ}}(q_i,r_i;v_i) \leq \sum_{i=0}^{t}\text{Pay}_{L^{\circ}}(q_i,r_i;v_i)$.
%Therefore by combing the two inequalities we have that $ \langle \rho (y),\sum_{i=0}^{t}r_i \rangle \leq \sum_{i=0}^{t}\text{Pay}_{L^{\circ}}(q_i,r_i;v_i)$.
%\end{proof}

However, proving other important properties such as bounded-worst case loss, information incorporation, and some form of incentive compatibility becomes challenging due to the non-convexity of the expected return 
\begin{equation*}
    \begin{aligned}
    \argmax_{r_t \in\reals^{d}} \quad &\underset{\text{Expected Payout}}{\underbrace{\langle \mu ,r_t \rangle}} - \underset{\text{Payment to Market}}{\underbrace{\text{Pay}_{L^{\circ}}(q_t,r_t;v_t)}}.
    \end{aligned}
\end{equation*}
We leave proving said results to future work. 
So, although this approach facilitates adaptive liquidity and no arbitrage, more work is required to justify its use.

%% file: conc.tex
\section{Conclusion}
\paragraph{Recap} In this work, we proposed a new prediction market framework based on the traditional DCFMM framework.
By doing so, the Smooth Quadratic Prediction Market satisfies many of the same axioms of the DCFMM while facilitating higher profits for the market maker.
Although the Smooth Quadratic Prediction does not satisfy the axiom of incentive compatibility, we show that the axiom is satisfied in an incremental sense and interestingly relate the methodology of general steepest descent with said behavior.
We also examined the Smooth Quadratic Prediction Market under the constraints of budget-bounded traders and buy-only markets.
Finally, we also presented introductory work on how the Smooth Quadratic Prediction Market could facilitate adaptive liquidity.
Although this work provides some core insights into the properties of the Smooth Quadratic Prediction Market, there are many future directions for this work.

\paragraph{Future Directions}
One direction would be to generalize this work beyond AD securities.
The original work of \cite{abernethy2013efficient} demonstrates how the DCFMM can be used for combinatorial and infinite space outcome securities.
Given the close relation between the design of the Smooth Quadratic Prediction Market and the DCFMM, we believe this generalization to more securities would be a feasible future direction. 
Another interesting future direction would be to analyze the convergence of the market when agents have varying beliefs. 
Using a random selection mechanism to select whose trades get processed, the analysis could be reduced to stochastic gradient descent.
Another direction would be to further prove/disprove axiom guarantees regarding our proposed approach to adaptive liquidity.
Finally, inspired by the shown equivalence of DCFMM and constant-function market makers (CFMMs) for asset exchanges by \cite{frongillo2023axiomatic}, one application would be to use the Smooth Quadratic Prediction Market to run an asset exchange.

%% file: apend.tex
%%%%%%%%%%%%%%%%%%%%%%%%%
%%%%%%%%%%%%%%%%%%%%%%%%%%

\begin{comment}
\section{Notation Tables}

\begin{table}[h]
	\centering
	\begin{tabular}{ll}
		Notation & Explanation \\
		\toprule
            $r_t \in \reals^d$ & Buy-sell bundle of goods \\
		\bottomrule
	\end{tabular}
\caption{Table for General Notation}\label{tab:cfpm}
\end{table}  
\end{comment}

%%%%%%%%%%%%%%%%%%%%%%%%%%%%

%%%%%%%%%%%%%%%%%%%%%%%%%%%%%%%

%\begin{definition}[Convex Function]
%        We say that a function $f:\reals^{d}\to (-\infty ,+\infty ]$ is a convex function if the epigraph of $f$, the set $\{(x,r) \mid x\in \dom (f), r\geq f(x) \}$,  is a convex set.
%\end{definition}

%\begin{lemma}[\cite{rockafellar1997convex}, Corr. 27.2.1]
%  Let $f$ be a closed proper convex function which has no direction of recession. Let $x_1,x_2,\dots$ be an sequence such that $\lim_{i\to\infty}f(x_i)=\inf f$. Then $x_1,x_2,\dots$, is a bounded sequence, and all its cluster points belong to the minimum set of $f$.
%\end{lemma}

%%%%%%%%%%%%%%%%%%%%%%%%%%%%

\section{Gradient Theorem of Line Integrals}

\begin{definition}[Vector field]
A vector field is a map $F:\reals^{d}\to\reals^{d}$.
\end{definition}

\begin{definition}[Conservative Vector Field]
If the vector field $F$ is the gradient of a function $f$, then $F$ is called a gradient or a conservative vector field. 
The function $f$ is called the potential or scalar of $F$.
\end{definition}

\begin{definition}[Differentiable Multivariable Function]\label{def:lineint}
The function $f:\reals^n \to \reals^m$ is differentiable at the point $a$
 if there exists a linear transformation $T:\reals^n \to \reals^m$ that satisfies the condition 
$$ \lim_{x\to a}\frac{||  f(x)-f(a)-T(x-a) ||}{||x-a||}=0 ~.$$
A function $f$ is said to be differentiable if $f$ is differentiable at all points within $\dom (f)$.
\end{definition}

\begin{definition}[Line integral of a Vector Field]\label{def:linevf}
For a vector field $F: U\subseteq \reals^{d}\to \reals^{d}$, the line integral along a piecewise smooth curve $s\subset U$, in the direction of $r$, is defined as
$$ \int_{s} F(r)\cdot dr = \int_{a}^{b} \langle F(r(t)),r'(t) \rangle dt $$
where $\cdot$ is the dot product and $r[a,b]\to s$ is a bijective parametrization of the curve $s$ such that $r(a)$ and $r(b)$ give the endpoints of $s$.
\end{definition}

\begin{theorem}[Gradient Theorem of Line Integrals]\label{def:gtli}
For $\phi : U\subseteq \reals^{d}
\to \reals$ as a differentiable function and $s$ as any continuous curve in $U$ which starts at a point $p$ and ends at a point $q$, then
$$\int_{s}\nabla \phi (r)\cdot dr= \phi ( q)-\phi (p) $$
where $\nabla \phi$ denotes the gradient vector field of $\phi$.
\end{theorem}

%%%%%%%%%%%%%%%%%%%%%%%%%%%%%

\section{Continuous FTRL}\label{app:cftrl}

A variety of works have examined FTRL in the continuous setting (CFTRL) such as \cite{kwon2017continuous,mertikopoulos2018cycles,cheung2021online}.
Let $\mathcal{C}\subseteq \reals^d$ denote a non-empty compact convex set.

\begin{definition}[Regularizer]
A convex function $R: \reals^{d} \to \reals \cup \{+\infty \}$ will be called a regularizer
function on $\mathcal{C}$ if $\dom (R)=\mathcal{C}$ and $R|_{\mathcal{C}}$ is strictly convex and continuous.
\end{definition}

For a given regularizer function $R$ on $\mathcal{C}$, we let $R_{max}=\max_{x\in\mathcal{C}}R(x)$ and $R_{min}=\min_{x\in\mathcal{C}}R(x)$. 

\begin{definition}[Choice Map]
The choice map associated to a regularizer function $R$ on $\mathcal{C}$ will be the map $Q_R : \reals^d \to \mathcal{C}$ defined as
$$ Q_R (y)=\argmax_{x\in\mathcal{C}}\{ \langle x,y  \rangle -R(x) \},\quad y\in \reals^d ~.$$
Note: $Q_R$ is the convex conjugate of $R$ and we have argmax since $R$ is strict convex and continuous hence has a well-defined unique sup. 
\end{definition}

In continuous time, instead of a sequence of payoff vectors $(u_n)_{n\in\mathbb{N}_{0}}$ in $\reals^d$, the agent will be facing a measurable and locally integrable stream of payoff vectors $(u_t)_{t\in\reals_{+}}$ in $\reals^d$.
Consider the process:

$$ x_{t}^{c} := Q_R ( \eta_t \int_{0}^{t}u_s ds )=\argmax_{x\in\mathcal{C}}\{ \langle x,\eta_t \int_{0}^{t}u_s ds  \rangle -R(x) \},  $$

where $(\eta_{t})_{t\in\reals_{+}}$ is a positive, nonincreasing and piecewise continuous parameter, while $x_{t}^{c} \in\mathcal{C}$ denotes the agent’s action at time $t$ given the history of payoff vectors $u_s$, $0\leq s < t$.
The dynamics of CFTRL can be expressed by
\begin{equation*}
    \begin{cases}
      y(t)=y(0)+\int_{0}^{t}u(x(s))ds \\
      x(t)= Q_R(y_t)
    \end{cases}\,.
\end{equation*}
It is worth noting that when $R$ is the entropy function, the dynamics reduce to replicator dynamics of evolutionary game theory.

\begin{theorem}[\cite{kwon2017continuous}, Theorem 4.1]\label{thm:contregret}
If  $R$ is a regularizer function on $\mathcal{C}$ and $(\eta_{t})_{t\in\reals_{+}}$ is a positive, non-increasing and piecewise continuous parameter, then, for every locally intergrable payoff stream $(u_t)_{t\in\reals_{+}}\in \reals^d $, we have:
$$\max_{x\in \mathcal{C}}\int_{0}^{t}\langle u_s ,x  \rangle ds  -\int_{0}^{t}\langle u_s,x_{s}^{c}\rangle ds \leq \frac{R_{max} - R_{min}}{\eta_t} $$
\end{theorem}

%%%%%%%%%%%%%%%%%%%%%%%%%%%%%%

\section{Constructing a DCFMM with AD securities via CFTRL}\label{app:cftrlequiv}

\begin{protocol}{Constructing a DCFMM with AD securities via Continuous FTRL}\label{proto:cftrltopred}
\textbf{Given}: regularizer $R:\reals^{d}\to \reals$, $\nabla C=R$, $\actionset =\Delta_d$, and $t_0=0$. \vspace{.1cm}\newline 
 
\textit{Repeat}:  
\begin{enumerate}
    \item Market receives trade bundle $r_t\in  \reals^{d}$
    \item Charge trader for bundle $r_t$:  $$C(q_{t}+r_t)-C(q_{t})=
    \int_{0}^{1}\langle  x_{t+1}^{c}(q_{t}+sr_t),r_t\rangle ds $$
    where $$x_{t+1}^{c}(q_{t}+sr_t)=\nabla C (q_{t}+sr_t)=\argmax_{p\in \Delta_d}\langle p, \int_{0}^{s} q_{t}+tr_t dt \rangle - R(p)$$ which is CFTRL.
    \item Set $q_{t+1}=q_{t}+r_t$ and t += 1
\end{enumerate}
\end{protocol}

\begin{theorem}
    By CFTRL defined via Protocol 1 with a fixed learning rate of $\eta=1$, the pricing of a bundle at a given state for a DCFMM can be expressed by an action taken by CFTRL.
    Thus, the worst-case loss of DCFMM is equivariant to the regret of CFTRL. 
\end{theorem}

\begin{proof}
    At any round, the market receives a bundle $r_t\in  \reals^{d}$ and runs CFTRL $$x_{t+1}^{c}(q_{t}+sr_t)=\nabla C (q_{t}+sr_t)=\argmax_{p\in \Delta_d}\langle p, \int_{0}^{s} q_{t}+tr_t dt \rangle - R(p)$$
    where the equality holds by \cite{kwon2017continuous}[Proposition 2].
    Also, observe that $$
    \underset{\text{parametric form}}{\underbrace{\int_{0}^{1}\langle  x_{t+1}^{c}(q_{t}+sr_t),r_t\rangle ds}}= C(q_{t}+r_t)-C(q_{t})$$
    where the equality holds by the Fundamental Theorem of Line Integrals.
    Thus, $$\underset{\text{Payments to Market}}{\underbrace{C(q_n)-C(q_0)}}=\sum_{t=0}^{n}C(q_t+r_t)-C(q_{t})= \underset{\text{CFTRL Loss}}{\underbrace{ \sum_{t=0}^{n}\int_{0}^{1}\langle  x_{t+1}^{c}(q_{t}+sr_t),r_t\rangle ds}}$$
    and that 
    $$\underset{\text{Worst-Case Market Payout}}{\underbrace{\max_{p \in \Delta_d}\langle p,q_n \rangle}}= \max_{p\in \Delta_d} \langle p,q_0+\sum_{t=0}^{n}r_t\rangle=\underset{\text{Worst-Case Action}}{\underbrace{ \max_{p\in \Delta_d}\sum_{t=0}^{n}\int_{0}^{1}\langle p,q_{t}+sr_{t}  \rangle ds }}~.$$
    Hence, via the equalities, the regret rate of CFTRL matches the worst-case loss of the DCFMM.
\end{proof}

\begin{corollary}
    Given a DCFMM defined via a strictly convex and differentiable $R$ defined over all of $\Delta_d$, then the worst-case loss is lower bounded by $R_{max}-R_{min}=\max_{p\in\rho(\Y )}R(p)-\min_{p\in\Delta_d}R(p)$.
\end{corollary}
Observe this worst-case loss matches the original derived worst-case loss in \cite{abernethy2013efficient}[Theorem. 4.4.] for DCFMM with AD securities (up to a difference of max and sup in terms of $R$, we leave generalizing in terms of sup to future work).

%%%%%%%%%%%%%%%%%%%%%%%%%%

%%%%%%%%%%%%%%%%%%%%%%%%%%%%%

\section{Gradient and Steepest Descent Results}\label{app:steep}

\begin{theorem}[\cite{ang2025}, Theorem 1]
    If $f$ is convex and $\ell_2$-based L-smooth, then for gradient descent where $x_{t+1}=x_t-\gamma \nabla f(x_t)$, $\|x_{k+1}-x^*\|_{2}^{2}<\|x_k-x^*\|_{2}^{2}$ if $x_k\notin \argmin f$ and $\|x_{k+1}-x^*\|_{2}^{2}=\|x_k-x^*\|_{2}^{2}$ otherwise.
\end{theorem}

\begin{lemma}[\cite{sidford2024}, Lemma 6.1.5]\label{sid615}
    For differentiable $f:\reals^d\to \reals$, norm $\|\cdot \|$, and $x\in \reals^n$ define $Q:\reals^n\to\reals$ as
    $$Q(y)=f(x)+\langle \nabla f(x),y-x \rangle+\frac{L}{2}\|x-y\|^2$$
    where $\|\cdot \|$ is an arbitrary norm and $L>0$.
    If $f$ is smooth w.r.t. $\|\cdot \|$ and $y^*=\argmin_{y\in\reals^d}Q(y)$ then $f(y^*)\leq f(x)-\frac{1}{2L}\|\nabla f(x)\|^{2}_{*}$. 
\end{lemma}

%\begin{lemma}[\cite{balles2020geometry}, Lemma 2]\label{lem:fdesc}
%Let $f$ be L-smooth w.r.t. $\|\cdot \|$. Then steepest descent w.r.t. $\|\cdot \|$ satisfies
%$$f(x_{t+1})\leq f(x_t)-\frac{1}{2L}\|\nabla f(x_t) \|^{2}_{*}~.$$
%\end{lemma}

\begin{theorem}[\cite{Wright_Recht_2022}, Theorem 3.5]\label{thm:accconv}
    Suppose that $f$ is bounded below and is L-smooth w.r.t. $\|\cdot \|$. Then all accumulation points $\overline{x}$ of the sequence $\{x_t\}$ generated by a scheme that satisfies  \begin{equation*}\label{eq:descentgrad}
        f(x_{t+1})\leq f(x_t)-\frac{1}{2L}\|\nabla f(x_t)\|^{2}_* 
    \end{equation*}
    are stationary, that is, $\nabla f(\overline{x})=0$. If in addition $f$ is convex, each such $\overline{x}$ is a solution of $\min_{x\in\reals^d}f(x)$. 
\end{theorem}

\begin{lemma}[\cite{sidford2024}, Lemma 6.1.7]\label{sid617}
    If $f:\reals^d\to \reals$ is differentiable and $\mu$-strong convex w.r.t. $\|\cdot \|$ then for all $x_*\in \argmin f$ and $x\in\reals^{n}$ it holds that 
    $$ \frac{\mu}{2}\|x-x_*\|^{2} \leq f(x)-f(x_*)\leq \frac{1}{2\mu}\|\nabla f(x)\|^{2}_{*}~.$$
\end{lemma}

\begin{lemma}
    If $f:\reals^d\to \reals$ is differentiable, $\mu$-strong convex w.r.t. $\|\cdot \|_{\mu}$, and L-smooth w.r.t. $\|\cdot \|_{L}$ then it holds that $\|y-x\|_{\mu}^{2}\leq \frac{L}{\mu}\|x-x_*\|^{2}_{L}$ where $x_*\in\argmin f$ and $x$ and $y$ are from Lemma \ref{sid615}.
\end{lemma}
\begin{proof}
    By Lemma \ref{sid616}, \ref{sid615}, and \ref{sid617}, we get $$\frac{\mu}{2}\|y-x_*\|_{\mu}^{2}\leq f(y)-f(x_*)\leq f(x)-f(x_*)\leq \frac{L}{2}\|x-x_*\|_{L}^{2}~.$$
    Hence, we get $\|y-x\|_{\mu}^{2}\leq \frac{L}{\mu}\|x-x_*\|^{2}_{L}$. 
\end{proof}

%%%%%%%%%%%%%%%%%%%%%%%%%%%%%%%%

\section{Smooth Quadratic Prediction Markets with Constraints Analysis}\label{app:cons}

%%%%%%%%%%%%%%%%%%%%%%%%%%%%%%%%

%\subsection{Budget-Bounded Traders}

\begin{figure}[!]
	\centering
	\begin{subfigure}{0.32\linewidth}
		\includegraphics[width=\linewidth]{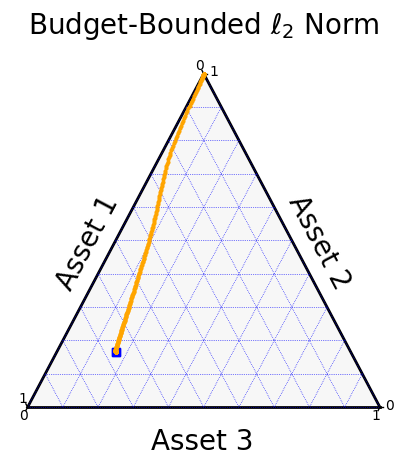}
		%\caption{Yellow color}
		\label{fig:subfigA}
	\end{subfigure}
	\begin{subfigure}{0.32\linewidth}
		\includegraphics[width=\linewidth]{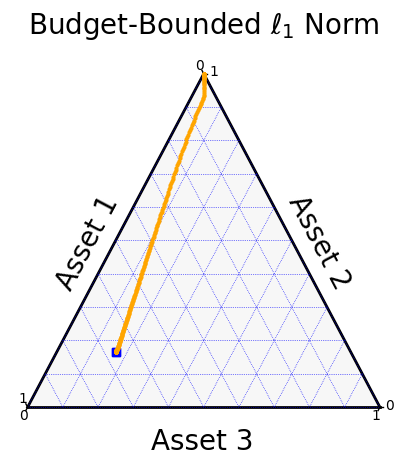}
		%\caption{Red color}
		\label{fig:subfigB}
	\end{subfigure}
	\begin{subfigure}{0.32\linewidth}
	        \includegraphics[width=\linewidth]{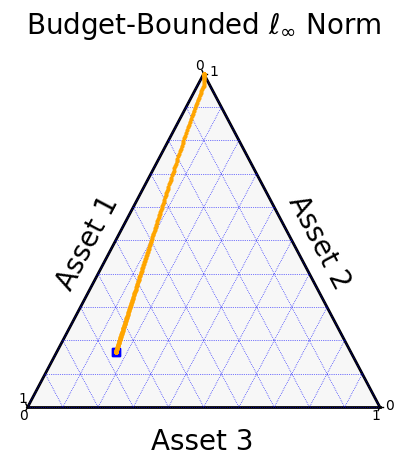}
	        %\caption{Green color}
	        \label{fig:subfigC}
         \end{subfigure}
	\caption{Let $q_{0}=(10,20,10)$, $C$ is softmax with smoothness of $L=1$, and $\mu = (1/6,1/6,2/3)$. The agents had a budget of $B=.01$. The blue square expresses $\mu$ and the orange path towards the blue square demonstrates the updating market distribution states. As denote by the titles's of each plot, we vary the norm used for the Smooth Quadratic Prediction Market. Note although softmax is not $\ell_{1}$-smooth, we use said norm experimentally for the sake of comparison.}
	\label{fig:subfiguresbudget}
\end{figure}

\textbf{Buy-Only Market}
For analysis sake, we substitute min argument of $r_t$ with $q_{t+1}=q_t+r_t$.
\begin{figure}[t]
	\centering
	\begin{subfigure}{0.32\linewidth}
		\includegraphics[width=\linewidth]{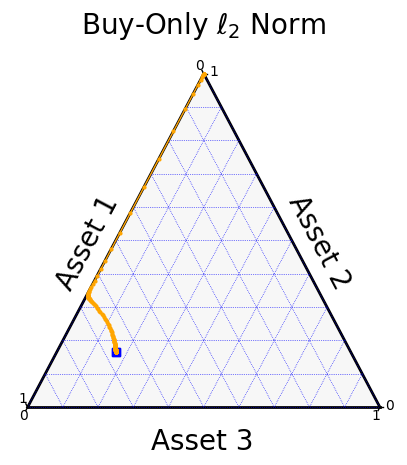}
		%\caption{Yellow color}
		\label{fig:subfigA}
	\end{subfigure}
	\begin{subfigure}{0.32\linewidth}
		\includegraphics[width=\linewidth]{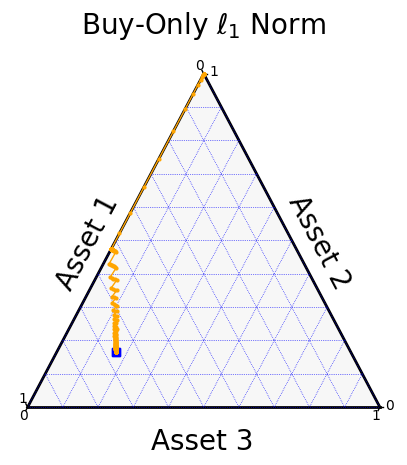}
		%\caption{Red color}
		\label{fig:subfigB}
	\end{subfigure}
	\begin{subfigure}{0.32\linewidth}
	        \includegraphics[width=\linewidth]{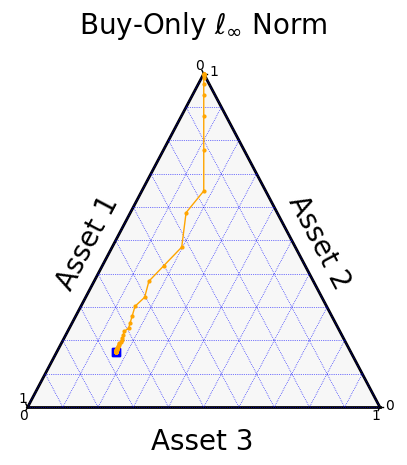}
	        %\caption{Green color}
	        \label{fig:subfigC}
         \end{subfigure}
	\caption{Let $q_{0}=(10,20,10)$, $C$ is softmax with smoothness of $L=1$, and $\mu = (1/6,1/6,2/3)$. The blue square expresses $\mu$ and the orange path towards the blue square demonstrates the updating market distribution states in a buy-only market. As denote by the titles's of each plot, we vary the norm used for the Smooth Quadratic Prediction Market. Note although softmax is not $\ell_{1}$-smooth, we use said norm experimentally for the sake of comparison.}
	\label{fig:subfiguresbuy}
\end{figure}

\begin{equation*}
\begin{aligned}
\min_{q_{t+1}\in\reals^{d}} \quad &  
\underset{\text{Payment to Market}}{\underbrace{\Big( \langle \nabla C(q_t),q_{t+1}-q_t \rangle +\frac{L}{2}\| q_{t+1}-q_t  \|^{2}_{2}  \Big)}} - \underset{\text{Expected Payout}}{\underbrace{\langle \mu ,q_{t+1}-q_t \rangle}} \\
\textrm{s.t.} \quad & q_t  \preceq q_{t+1} \\
\end{aligned}
\end{equation*}

Let $r= q_{t+1}-q_t$ and $c=\nabla C(q_t)-\mu$. 
This simplifies the objective too 

\begin{equation*}
\begin{aligned}
\min_{r} \quad &  
 \langle c ,r \rangle +\frac{L}{2}\| r  \|^{2}_{2}  \\
\textrm{s.t.} \quad & -r  \preceq 0. \\
\end{aligned}
\end{equation*}

Let the Lagrangian be denoted by $\mathcal{L}(r,\lambda )=\langle c,r \rangle +\|r\|_{2}^{2}-\langle \lambda, r \rangle $ where $\lambda \in \reals^{d}_+$. 
Solving for $r$ from $\min_{r\in\reals^{d}} \mathcal{L}(r,\lambda )$, we get $r^*=\frac{1}{L}(\lambda - c)$. 
Note the dual problem is 
$$\max_{\lambda \succeq 0 } \frac{-1}{2L}\langle c,c\rangle - \frac{1}{2L}\langle \lambda, \lambda \rangle +\frac{1}{L}\langle \lambda ,c \rangle  ~.$$
We claim that $\lambda^*=c_+$.
We verify $r^*=\frac{1}{L}(c_+ - c)$ and $\lambda^*=c_+$ via KKT conditions. 

\begin{itemize}
    \item \textbf{Stationarity} ($\frac{\partial \mathcal{L}}{\partial r^*_{i}}(r^*,\lambda^*)=0$ $\forall \; i\in[d]$): 
    $c+c_+-c-c_+=0$
    \item \textbf{Complimentary Slackness} ($-\lambda^*_ir^*_i=0$ $\forall \; i\in[d]$): If $c_i>0$ then we have $\frac{-c_i}{L}(c_i-c_i)=0$ and if $c_i\leq 0$ we have $\frac{-0}{L}(c_i-0)=0$.
    \item \textbf{Primal Feasibility}: If $c_i\geq 0$ then we have $\frac{1}{L}(c_i-c_i)=0$ and if $c_i <0$ we have $c_i <0$.
    \item \textbf{Dual Feasibility}: $0 \preceq c_+$ by definition 
\end{itemize}

Hence, we get that the update step is
$$q_{t+1}=q_t+\frac{1}{L}( (\nabla C(q_t)-\mu )_+- (\nabla C(q_t)-\mu ))$$
and observe when $\nabla C(q_t)=\mu$ we have a stationary point.

%%%%%%%%%%%%%%%%%%%%%%%%%%%%%%%

\section{Smooth Quadratic Prediction Markets with Adaptive Liquidity}\label{app:liquid}

\begin{definition}[Asymmetric norm]\label{def:asymnorm}
A function $g:\reals^{d}\to\reals_+$ is an asymmetric norm if it satisfies $\forall $ $x,y\in \reals^{d}$:
\begin{itemize}
    \item Non-negativity: $g(x) \geq 0$
    \item Definiteness: $g(x) = g(-x) = 0$ if and only if $x = 0$
    \item Positive homogeneity: $g(\alpha x) = \alpha g(x)$ for all $\alpha  > 0$
    \item Triangle inequality: $g(x + y) \leq g(x) + g(y)$.
\end{itemize}
\end{definition}

The following inequality due to the smoothness of $C^{\circ}$ motivates our proposed payment of $\text{Pay}_{L^{\circ}}$

\begin{align*}
\text{Pay}_{D^{\circ}}(q_t,r_t;v_t) &= C^{\circ}(q_t+r_t;v_t+g(r_t))-C^{\circ}(q_t;v_t) \\
&= \underset{\text{Payment to Breg. Market}}{\underbrace{\Big( C^{\circ}(q_t+r_t;v_t+g(r_t))-C^{\circ}(q_t;v_t+g(r_t))  \Big)}}
- \underset{\text{Liquidity Fee}\geq 0}{\underbrace{\Big( C^{\circ}(q_{t};v_t+g(r_t))-C^{\circ}(q_t;v_t)\Big)}}\\
& = \underset{\text{Payment to Breg. Market}}{\underbrace{\Big(  \langle \nabla C^{\circ}(q_t;v_t+g(r_t))  \rangle +D_{C^{\circ}(\cdot ; v_t+g(r_t))}(q_t+r_t,q_t)\Big)}} - \underset{\text{Liquidity Fee}\geq 0}{\underbrace{\Big( C^{\circ}(q_{t};v_t+g(r_t))-C^{\circ}(q_t;v_t)\Big)}}  \\
&\leq \underset{\text{Payment to Smooth-Quad Market}}{\underbrace{\Big( \langle \nabla C^{\circ}(q_t;v_t+g(r_t)),r_t \rangle +\frac{L^{\circ}}{2}\|r_t\|^{2}  \Big)}}+
\underset{\text{Liquidity Fee}\geq 0}{\underbrace{\Big( C^{\circ}(q_{t};v_t+g(r_t))-C^{\circ}(q_t;v_t)\Big)}}\\
&= \text{Pay}_{L^{\circ}}(q_t,r_t;v_t).
\end{align*}

  \adaparb*
\begin{proof}
    By (Lemma 3.6, \cite{abernethy2014general}) the VPM-DCFMM satisfies no arbitrage, i.e, for all $r_0,\dots ,r_t\in\textbf{S}$ it holds that $\langle \rho (y),\sum_{i=0}^{t}r_i \rangle \leq \sum_{i=0}^{t} \text{Pay}_{D^{\circ}}(q_i,r_i;v_i)$ for some $y\in\Y $.
    However, it also holds that $\sum_{i=0}^{t} \text{Pay}_{D^{\circ}}(q_i,r_i;v_i) \leq \sum_{i=0}^{t}\text{Pay}_{L^{\circ}}(q_i,r_i;v_i)$.
Therefore by combing the two inequalities we have that $ \langle \rho (y),\sum_{i=0}^{t}r_i \rangle \leq \sum_{i=0}^{t}\text{Pay}_{L^{\circ}}(q_i,r_i;v_i)$.
\end{proof}